%% file: neurips_2022.tex
\documentclass{article}

    \PassOptionsToPackage{numbers}{natbib}

\usepackage[preprint]{neurips_2022}

\usepackage[utf8]{inputenc} 
\usepackage[T1]{fontenc}    
\usepackage{hyperref}       
\usepackage{url}            
\usepackage{booktabs}       
\usepackage{amsfonts}       
\usepackage{nicefrac}       
\usepackage{microtype}      
\usepackage{xcolor}         
\usepackage[pdftex]{graphicx}
\usepackage{amsfonts}
\usepackage{amsmath}
\usepackage{amssymb}
\usepackage{mathtools}
\usepackage{amsthm}
\usepackage{cleveref}
\usepackage{algorithm}
\usepackage{algpseudocode}

\crefname{lemma}{Lemma}{Lemmas}
\crefname{remark}{Remark}{Remark}
\crefname{theorem}{Theorem}{Theorem}
\crefname{proposition}{Proposition}{Proposition}

\newtheorem{remark}{Remark}[section]
\newtheorem{conjecture}{Conjecture}[section]
\theoremstyle{definition}
\newtheorem{definition}{Definition}[section]

\title{AutoInit: Automatic Initialization via Jacobian Tuning}

\newcommand*\samethanks[1][\value{footnote}]{\footnotemark[#1]}

\author{%
    Tianyu He \thanks{The first two authors contributed equally to this work, and Tianyu He is the corresponding author.}\\
    Brown University \\
    \texttt{tianyu\_he@brown.edu}
    \And
    Darshil Doshi \samethanks\\
    Brown University \\
    \texttt{darshil\_doshi@brown.edu}
    \And
    Andrey Gromov \\
    Brown University\\
    \texttt{andrey\_gromov@brown.edu}
}

\begin{document}

\maketitle

\begin{abstract}
    Good initialization is essential for training Deep Neural Networks (DNNs). Oftentimes such initialization is found through a trial and error approach, which has to be applied anew every time an architecture is substantially modified, or inherited from smaller size networks leading to sub-optimal initialization. In this work we introduce a new and cheap algorithm, that allows one to find a good initialization automatically, for general feed-forward DNNs. The algorithm utilizes the Jacobian between adjacent network blocks to tune the network hyperparameters to criticality. We solve the dynamics of the algorithm for fully connected networks with ReLU and derive conditions for its convergence. We then extend the discussion to more general architectures with BatchNorm and residual connections. Finally, we apply our method to ResMLP and VGG architectures, where the automatic one-shot initialization found by our method shows good performance on vision tasks.  
\end{abstract}

\section{Introduction}
Initializing Deep Neural Networks (DNNs) correctly is crucial for trainability and convergence. In the recent years, there has been remarkable progress in tackling the problem of exploding and vanishing gradients. 
One line of work utilizes the convergence of DNNs to Gaussian Processes in the limit of infinite width \citep{neal1996priors, lee2018deep, matthews2018gaussian, novak2018bayesian, garriga2018deep, hron2020infinite, yang2019tensor}. The infinite width analysis is then used to determine critical initialization for the hyperparameters of the network \citep{he2015delving, poole2016exponential, schoenholz2016deep, lee2018deep, roberts2021principles, doshi2021critical}. 
It has further been shown that dynamical isometry can improve the performance of DNNs \citep{pennington2018spectral, xiao2018dynamical}. Exploding and vanishing gradients can also be regulated with special activation functions such as SELU \citep{klambauer2017self-normalizing} and GPN \citep{lu2020bidirectionally}. Deep Kernel shaping \citep{martens2021deep, zhang2022deep} improves trainability of deep networks by systematically controlling $Q$ and $C$ maps.
Normalization layers such as LayerNorm \citep{ba2016layer}, BatchNorm \citep{ioffe2015batch} and \citep{wu2018group} facilitate training of DNNs by significantly enhancing the critical regime \citep{doshi2021critical}. There have also been algorithmic attempts at regulating the forward pass, such as LSUV \citep{mishkin2015lsuv}.

Another line of work sets the networks with residual connections to criticality by suppressing the contribution from the residual branches at initialization.
In Highway Networks \citep{srivastava2015training}, this is achieved by initializing the network to have a small ``transform gate''. \citet{goyal2017accurate} achieve this in ResNets, by initializing the scaling coefficient for the residual block's last BatchNorm at 0. In Fixup \citep{zhang2019fixup} and T-Fixup \citep{huang2020improving}, careful weight-initialization schemes ensure suppression of residual branches in deep networks. Techniques such as SkipInit \citep{de2020batch}, LayerScale \citep{touvron2021cait} and ReZero \citep{bachlechner2021rezero} multiply the residual branches by a trainable parameter, initialized to a small value or to 0.

Despite this progress, the aforementioned techniques are limited by either the availability of analytical solutions, specific use of normalization layers, or the use of residual connections. One needs to manually decide on the techniques to be employed on a case-by-case basis. In this work, we propose a simple algorithm, which we term $\texttt{AutoInit}$, that automatically initializes a DNN to criticality. Notably, the algorithm can be applied to any feedforward DNN, irrespective of the architectural details, large width assumption or existence of analytic treatment. We expect that $\texttt{AutoInit}$ will be an essential tool in architecture search tasks because it will always ensure that a never-before-seen architecture is initialized well.

\subsection{Criticality in Deep Neural Networks}
In the following, we employ the definition of criticality using \emph{Partial Jacobian} \citep{doshi2021critical}.
Consider a DNN made up of a sequence of blocks. Each block consists of Fully Connected layers, Lipschitz activation functions, Convolutional layers, Residual Connections, LayerNorm\citep{ba2016layer}, BatchNorm\citep{ioffe2015batch}, AffineNorm\citep{touvron2021resmlp}, LayerScale\citep{touvron2021cait}, or any combination of thereof. We consider a batched input to the network, where each input tensor $x \in \mathbb{R}^{n^0_1} \otimes \mathbb{R}^{n^0_2} \otimes \cdots$ is taken from the batch $B$ of size $\lvert B \rvert$. The output tensor of the $l^{th}$ block is denoted by $h^l (x) \in \mathbb{R}^{n^l_1} \otimes \mathbb{R}^{n^l_2} \otimes \cdots$. $h^{l+1}(x)$ depends on $h^l(x)$ through a layer-dependent function $\mathcal{F}^{l}$, denoting the operations of the aforementioned layers. This function, in turn, depends on the parameters of the various layers within the block, denoted collectively by $\theta^{l+1}$. The explicit layer dependence of the function $\mathcal{F}^{l}$ highlights that we do not require the network to have self-repeating layers (blocks). We note that $h^{l+1} (x)$ can, in general, depend on $h^{l} (x')$ for all $x'$ in the batch $B$; which will indeed be the case when we employ BatchNorm. The recurrence relation for such a network can be written as
\begin{align}\label{eq:DNNrecursion}
    h^{l+1} (x) = \mathcal{F}^{l+1}_{\theta^{l+1}} \left( \{h^l(x') \;|\; \forall x' \in B \} \right) \,, 
\end{align}
where we have suppressed all the indices for clarity. Each parameter matrix $\theta^{l+1}$ is sampled from a zero-mean distribution. We will assume that some $2+\delta$ moments of $|\theta^{l+1}|$ are finite such that the Central Limit Theorem holds. Then the variances of $\theta^{l+1}$ can be viewed as hyperparameters and will be denoted by $\sigma^{l+1}_{\theta}$ for each $\theta^{l+1}$.

We define the $\texttt{Flatten}$ operation, which reshapes the output $h^l(x)$ by merging all its dimensions.
\begin{align}
    \bar h^l(x) = \texttt{Flatten}\left( h^l(x) \right) \sim \mathbb{R}^{N^l} \,,
\end{align}
where $N^l \equiv n^l_1 n^l_2 \cdots$.

\begin{definition}[Average Partial Jacobian Norm (APJN)]
\label{def:APJN}
    For a DNN given by \eqref{eq:DNNrecursion}, APJN is defined as 
    \begin{align}
        \mathcal J^{l_0, l} \equiv \mathbb E_{\theta} \left[\frac{1}{|B| N_l} \sum_{j=1}^{N_{l}} \sum_{i=1}^{N_{l_0}} \sum_{x, x' \in B} \frac{\partial \bar{h}^{l}_j(x')}{\partial \bar{h}^{l_0}_i(x)} \frac{\partial \bar{h}^{l}_j(x')}{\partial \bar{h}^{l_0}_i(x)} \right] \,,
    \end{align}
    where $\mathbb E_\theta[\cdot]$ denotes the average over parameter initializations.
\end{definition}

\begin{remark}
    For DNNs without BatchNorm and normalized inputs, definition of APJN for $|B|>1$ is equivalent to the one in $|B|=1$ case.
\end{remark}

We use APJN as the empirical diagnostic of criticality. 

\begin{definition}[Critical Initialization]
\label{def:critical}
    A DNN given by \eqref{eq:DNNrecursion}, consisting of $L+2$ blocks, including input and output layers, is critically initialized if all block-to-block APJN are equal to $1$, i.e.
    \begin{align}
        \mathcal J^{l,l+1} = 1 \,, \quad  \forall \quad 1 \leq l \leq L \,.
    \end{align}
\end{definition}

Critical initialization as defined by \Cref{def:critical} is essential, as it prevents the gradients from exploding or vanishing at $t=0$. One can readily see this by calculating the gradient for any flattened parameter matrix $\theta$ at initialization:
\begin{align}\label{eq:grad}
    \frac{1}{|\theta^l|}\|\nabla_{\theta^l} \mathcal L \|^2_2 =& \frac{1}{|\theta^l|} \left\|\sum_{\mathrm{all}} \frac{\partial \mathcal L}{\partial \bar{h}^{L+1}_i} \frac{\partial \bar{h}^{L+1}_i}{\partial \bar{h}^L_j}\cdots \frac{\partial \bar{h}^{l+1}_k}{\partial \bar{h}^l_m} \frac{\partial \bar{h}^l_m}{\partial \theta^{l}_{n}} \right\|^2_2 \nonumber \\
    \sim &\, O \left( \frac{1}{|\theta^l|} \left\|\frac{\partial \mathcal L}{\partial \bar{h}^{L+1}_i} \right\|^2_2 \cdot \mathcal J^{L, L+1} \cdots \mathcal J^{l, l+1} \cdot \left \|\frac{\partial \bar{h}^l_m}{\partial \theta^{l}_{n}} \right\|^2_F \right)\,,
\end{align}
where $\| \cdot \|_F$ denotes the Frobenius norm. In the second line, we utilized the factorization property of APJN
\begin{align}\label{eq:factor}
    \mathcal J^{l_0,l} = \prod_{l'=l_0}^{l-1} \mathcal J^{l', l'+1} \,,
\end{align}
which holds in the infinite width limit given there is no weight sharing across the blocks.

One may further require $\left\| \partial \mathcal L / \partial \bar{h}^{L+1}_i \right\|_2 \sim O(1)$. However, in practice we observe that this requirement is less important once the condition in \Cref{def:critical} is met.

\subsection{Automatic Critical Initialization}
For general architectures, analytically calculating APJN is often difficult or even impossible. This poses a challenge in determining the correct parameter initializations to ensure criticality; especially in networks without self-similar layers. Moreover, finite network width is known to have nontrivial corrections to the criticality condition \citep{roberts2021principles}. This calls for an algorithmic method to find critical initialization. 

To that end, we propose the \Cref{alg:j_general} that we called $\texttt{AutoInit}$ for critically initializing deep neural networks \emph{automatically}, without the need for analytic solutions of the signal propagation or of the meanfield approximation. The algorithm works for general feedforward DNNs, as defined in \eqref{eq:DNNrecursion}. Moreover, it naturally takes into account all finite width corrections to criticality because it works directly with an instance of a network. We do tacitly assume the existence of a critical initialization. If the network cannot be initialized critically, the algorithm will return a network that can propagate gradients well because the APJNs will be pushed as close to $1$ as possible.

The central idea behind the algorithm is to choose the hyperparameters for all layers such that the condition in \Cref{def:critical} is met. This is achieved by optimizing a few auxiliary scalar parameters $a^l_{\theta}$ of a twin network with parameters $a^l_{\theta} \theta^{l}$ while freezing the parameters $\theta^{l}$. The loss function is minimized by the condition mentioned in \Cref{def:critical}. 
\begin{algorithm}[h]
    \caption{\texttt{AutoInit} (SGD)}
    \label{alg:j_general}
    \begin{algorithmic}
        \State {\textbf{Input:}} Model $\mathcal M(\{\sigma^l_\theta;\, a^l_{\theta}(t) \; | \; \forall \; 1 \leq l \leq L\,, \forall \theta^l \})$, Loss function $\mathcal L(\{\mathcal J^{l,l+1}\}_{l=1}^{L})$, $T$, $\epsilon$, and $\eta$.
        \State \textbf{Set} $t=0$ and $\{a_\theta^l(0)=1\}$
        \State \textbf{Evaluate} $\mathcal L(0)$
        \While {$0 \leq t < T$ and $\mathcal L(t) > \epsilon$}
        \State $a^l_\theta(t+1) = a^l_\theta(t) - \eta \nabla_{a^l_\theta} \mathcal L(t)$
        \State \textbf{Evaluate} $\mathcal L(t+1)$
        \EndWhile
        \State \textbf{Return} $\mathcal M(\{\sigma^l_\theta = \sigma^l_{\theta} a^l_{\theta}(t) ;\, 1\; | \; \forall \; 1 \leq l \leq L\,, \forall \theta^l \})$
    \end{algorithmic}
\end{algorithm}

In practice, for speed and memory reasons we use an unbiased estimator \citep{hoffman2019robust} of APJN in \Cref{alg:j_general}, defined as
\begin{align}\label{eq:j_est}
    \hat {\mathcal{J}}^{l, l+1} \equiv \frac{1}{N_v} \sum_{\mu=1}^{N_v} \left[\frac{1}{|B| N_l} \sum_{j=1}^{N_{l+1}} \sum_{k=1}^{N_{l+1}} \sum_{i=1}^{N_{l}} \sum_{x, x' \in B} \frac{\partial (v_{\mu j} \bar{h}^{l+1}_j(x'))}{\partial \bar{h}^{l}_i(x)} \frac{\partial (v_{\mu k} \bar{h}^{l+1}_k(x'))}{\partial \bar{h}^{l}_i(x)} \right] \,,
\end{align}
where each $v_{\mu i}$ is a unit Gaussian random vector for a given $\mu$. The Jacobian-Vector Product (JVP) structure in the estimator speeds up the computation by a factor of $N_{l+1} / N_v$ and consumes less memory at the cost of introducing some noise. 

In \Cref{sec:auto} we analyze $\texttt{AutoInit}$ for multi-layer perceptron (MLP) networks. Then we discuss the problem of exploding and vanishing gradients of the tuning itself; and derive bounds on the learning rate for ReLU or linear MLPs. In \Cref{sec:bn} we extend the discussion to BatchNorm and provide a strategy for using $\texttt{AutoInit}$ for a general network architecture. In \Cref{sec:exp} we provide experimental results for more complex architectures: VGG19\_BN and ResMLP-S12.

\section{AutoInit for MLP networks}
\label{sec:auto}
MLPs are described by the following recurrence relation for preactivations
\begin{align}\label{eq:mlp_preact}
    h^{l+1}_i(x) = \sum_{j=1}^{N_l} W^{l+1}_{ij} \phi(h^l_j(x)) + b^{l+1}_i  \,.
\end{align}
Here $x$ is an input vector, weights $W^{l+1}_{ij} \sim \mathcal N(0, \sigma_w^2/N_l)$ and biases $b^{l+1}_i \sim \mathcal N(0, \sigma_b^2)$ are collectively denoted as $\theta^{l+1}$. We assume $\phi$ is a Lipschitz activation function throughout this paper.

For a network with $L$ hidden layers, in infinite width limit $N_l \rightarrow \infty$, preactivations \{$h^l_i(x) \,|\, 1 \leq l \leq L, \forall i \in N_l\}$ are Gaussian Processes (GPs). The distribution of preactivations is then determined by the Neural Network Gaussian Process (NNGP) kernel
\begin{align}
    \mathcal K^{l}(x, x') = \mathbb E_{\theta} \left[ h^l_i(x) h^l_i(x') \right] \,,
\end{align}
which value is independent of neuron index $i$. The NNGP kernel can be calculated recursively via
\begin{align}
    \mathcal K^{l+1}(x, x') = \sigma_w^2 \mathbb E_{h_i^l(x), h_i^l(x') \sim \mathcal N(0, \mathcal K^l(x, x'))} \left[\phi\left(h_i^l(x)\right) \phi\left(h_i^l(x')\right) \right] + \sigma_b^2 \,.
\end{align}
Note that we have replaced the average over parameter initializations $\mathbb{E}_\theta[\cdot]$ with an average over preactivation-distributions $\mathbb E_{h_i^l(x), h_i^l(x') \sim \mathcal N(0, \mathcal K^l(x, x'))} [\cdot]$; which are interchangeable in the infinite width limit \citep{lee2018deep, roberts2021principles}. Critical initialization of such a network is defined according to \Cref{def:critical}.

In practice, we define a twin network with extra parameters, for MLP networks the twin preactivations can be written as
\begin{align}\label{eq:twin_preact}
    h^{l+1}_i(x) = \sum_{j=1}^{N_l} a_W^{l+1} W^{l+1}_{ij} \phi(h^l_j(x)) + a_b^{l+1} b^{l+1}_i \,,
\end{align}
where $a_{\theta}^{l+1} \equiv \{a^{l+1}_W, a^{l+1}_b\}$ are auxiliary parameters that will be tuned by \Cref{alg:j_train}. 
\begin{algorithm}[h]
    \caption{\texttt{AutoInit} for MLP (SGD)}
    \label{alg:j_train}
    \begin{algorithmic}
        \State {\textbf{Input:}} Model $\mathcal M(\{\sigma_w, \sigma_b, a_W^l(t), a_b^l(t) \;| \; \forall 1 \leq l \leq L \})$, Loss function $\mathcal L(\{\mathcal J^{l,l+1}\}_{l=1}^{L})$, $T$, $\epsilon$, and $\eta$.
        \State \textbf{Set} $t=0$, $\{a_W^l(0)=1\}$ and $\{a_b^l(0)=1\}$
        \State \textbf{Evaluate} $\mathcal L(0)$
        \While {$0 \leq t < T$ and $\mathcal L(t) > \epsilon$}
        \State $a^l(t+1) = a^l(t) - \eta \nabla_{a^l} \mathcal L(t)$
        \State \textbf{Evaluate} $\mathcal L(t+1)$
        \EndWhile
        \State \textbf{Return} $\mathcal M(\{a^l_W(t) \sigma_w, \, a^l_b(t) \sigma_b, 1,\, 1 \;| \; \forall 1 \leq l \leq L \})$
    \end{algorithmic}
\end{algorithm}

In \Cref{alg:j_train}, one may also return  $\mathcal M(\{\sigma_w, \sigma_b, a_W^l(t), a_b^l(t) \;| \; \forall 1 \leq l \leq L \})$, while freezing all $a^l_{\theta}$. However, this leads to different training dynamics while updating weights and biases. Alternatively, one can leave the auxiliary parameters trainable, but in practice this leads to unstable training dynamics.

\paragraph{Loss function} The choice of loss function $\mathcal L$ is important. We will use the following loss
\begin{align}\label{eq_loss_sq_J}
    \mathcal L_{\log} = \frac{1}{2} \sum_{l=1}^L \left[\log(\mathcal J^{l, l+1})\right]^2 \,,
\end{align}
We will refer to \eqref{eq_loss_sq_J} as Jacobian Log Loss (JLL).

This definition is inspired by the factorization property \eqref{eq:factor}, which allows one to optimize each of the partial Jacobian norms independently. Thus the tuning dynamics is less sensitive to the depth. One could naively use $\log(\mathcal J^{0, L+1})^2$ as a loss function, however optimization will encounter the same level of exploding or vanishing gradients problem as \eqref{eq:grad}. One may worry that the factorization property will be violated for $t>0$, due to the possible correlation across all $\{a^l(t)\}$. It turns out that the correlation introduced by \Cref{alg:j_train} does not change the fact that all weights and biases are iid, ensuring that \eqref{eq:factor} holds for any $t \geq 0$.

Another choice for the loss is Jacobian Square Loss (JSL), defined as $\mathcal L_2 = \frac{1}{2} \sum_{l=1}^L \left(\mathcal J^{l, l+1}  - 1 \right)^2$. However JSL has poor convergence properties when $\mathcal J^{l, l+1} \gg 1$. One may further restrict the forward pass by adding terms that penalize the difference between $\mathcal K^l(x, x)$ and $\mathcal K^{l+1}(x,x)$. For brevity, we leave these discussions for the appendix.

\paragraph{Exploding and Vanishing Gradients} While the objective of \Cref{alg:j_train} is to solve the exploding and vanishing gradients problem, the \Cref{alg:j_train} itself has the same problem, although not as severe. 

Consider optimizing MLP networks using $\mathcal L_{\log}$, where the forward pass is defined by \eqref{eq:twin_preact}. Assuming the input data $x$ is normalized, the SGD update (omit x) of $a^{l}_{\theta}$ at time $t$ can be written as
\begin{align}\label{eq:a_update} 
    a_{\theta}^{l+1}(t+1) - a_{\theta}^{l+1}(t) = - \eta \sum_{l' \geq l}^L  \frac{\partial \log \mathcal J^{l', l'+1}(t)}{\partial a_{\theta}^{l+1}(t)} \log \mathcal J^{l', l'+1}(t) 
\end{align}
For a deep neural network, i.e. $|L - l| \gg 1$ holds for some $l$, the depth dependent term of \eqref{eq:a_update} can lead to exploding or vanishing gradients problems. We will show next that this is not the familiar exploding or vanishing gradients problem.

First we would like to explain the vanishing gradients problem for $a_W^{l+1}$. Rewrite the right hand side of \eqref{eq:a_update} as 
\begin{align}\label{eq:iso}
    - \eta \sum_{l' \geq l}^L  \frac{\partial \log \mathcal J^{l', l'+1}(t)}{\partial a_{W}^{l+1}(t)} \log \mathcal J^{l', l'+1}(t) = - \eta \frac{2}{a_W^{l+1}(t)} \log \mathcal J^{l, l+1}(t) + (l' > l\; \mathrm{terms}) \,.
\end{align}
Vanishing gradients can only occur if the isolated term is exactly canceled by the other terms for all $t \geq 0$, which does not happen in practice. 

To discuss the exploding gradients problem for $a_W^{l+1}$ we consider the update of $a_W^{l+1}$ (omit t). Depth dependent terms can be written as
\begin{align}\label{eq:tauto_aw} 
    \sum_{l'>l}^L \frac{\partial \log \mathcal J^{l',l'+1}}{\partial a_W^{l+1}} \log \mathcal J^{l', l'+1} = \sum_{l'>l}^L \left(\frac{4\chi^{l'}_{\Delta}}{a_W^{l+1} \mathcal J^{l', l'+1}}  \chi^{l'-1}_{\mathcal K} \cdots \chi^{l+1}_{\mathcal K} \mathcal K^{l+1}(x, x)\right) \log \mathcal J^{l', l'+1} \,,
\end{align}
where we have defined two new quantities $\chi^{l'}_{\Delta} \equiv (a_W^{l'+1} \sigma_w)^2 \mathbb E_{\theta} \left[\phi''(h^{l'}_i) \phi''(h^{l'}_i) + \phi'(h^{l'}_i) \phi'''(h^{l'}_i) \right]$ and $\chi^{l'}_{\mathcal K} \equiv (a_W^{l'+1} \sigma_w)^2 \mathbb E_{\theta} \left[\phi'(h^{l'}_i) \phi'(h^{l'}_i) + \phi(h^{l'}_i) \phi''(h^{l'}_i) \right]$. We note that the exploding gradients problem for $a_W^{l+1}$ in $\texttt{AutoInit}$ is not severe for commonly used activation functions:
\begin{itemize}
    \item $\tanh$-like bounded odd activation functions: $\chi_{\mathcal K}^{l'} \leq 1$ holds and $\mathcal K^l(x,x)$ saturates to a constant for large $l$. Thus the divergence problem of \eqref{eq:tauto_aw} is less severe than the one of \eqref{eq:grad} when $\mathcal J^{l', l'+1} > 1$.
    \item $\mathrm{ReLU}$: $\chi^{l'}_{\Delta}=0$.
    \item $\mathrm{GELU}$: The sum in \eqref{eq:tauto_aw} scales like $O(L \prod_{\ell=1}^L \chi^{\ell}_{\mathcal K})$ for large $L$, which may lead to worse exploding gradients than \eqref{eq:grad} for a reasonable $L$. Fortunately, for $\chi^{l'}_{\mathcal K} > 1$ cases, $\chi_{\Delta}^{l'}$ is close to zero. As a result, we find numerically that the contribution from \eqref{eq:tauto_aw} is very small.
\end{itemize}

For $a_b^{l+1}$, there is no isolated term like the one in \eqref{eq:iso}. Then the update of $a_b^{l+1}$ is proportional to
\begin{align}\label{eq:tauto_ab}
    \sum_{l'>l}^L \frac{\partial \log \mathcal J^{l',l'+1}}{\partial a_b^{l+1}} \log \mathcal J^{l, l+1} = \sum_{l'>l}^L \left(\frac{4 a_b^{l+1}}{\mathcal J^{l', l'+1}} \chi^{l'}_{\Delta} \chi^{l'-1}_{\mathcal K} \cdots \chi^{l+1}_{\mathcal K} \sigma_b^2 \right) \log \mathcal J^{l, l+1} \,.
\end{align}
Comparing \eqref{eq:tauto_ab} and \eqref{eq:tauto_aw}, it is clear that the exploding gradients problem for $a_b^{l+1}$ is the same as that for $a_W^{l+1}$, hence not severe for common activation functions. The vanishing gradients problem is seemingly more serious, especially for $\sigma_b=0$. However, the vanishing gradients for $a_b^{l+1}$ does not prevent \texttt{AutoInit} from reaching a critical initialization: 
\begin{itemize}
    \item For $\sigma_b > 0$, as $a_W^{l+1}$ gets updated, the update in \eqref{eq:tauto_ab} gets larger with time t.
    \item For $\sigma_b=0$ the phase boundary is at $\sigma_w \geq 0$, which can be reached by $a_W^{l+1}$ updates.
\end{itemize}

\subsection{Linear and ReLU networks}
In general, it is hard to predict a good learning rate $\eta$ for the \Cref{alg:j_train}. However, for ReLU (and linear) networks, we can estimate the optimal learning rates. We will discuss ReLU in detail. Since $a_b^l$ can not receive updates in this case, we only discuss updates for $a_W^l$. 

Different APJN $\{\mathcal J^{l,l+1}\}$ for ReLU networks evolve in time independently according to
\begin{align}\label{eq:relu_jupdate}
    \sqrt{\mathcal J^{l,l+1}(t+1)} - \sqrt{\mathcal J^{l,l+1}(t)} = -\eta \frac{\sigma_w^2}{\sqrt{\mathcal J^{l,l+1}(t)}} \log \mathcal J^{l,l+1}(t) \,.
\end{align}

Then one can show that for any time $t$:
\begin{align}\label{eq:eta_t}
    \eta_t < & \min_{1 \leq l \leq L} \left\{\frac{2\left( \sqrt{\mathcal J^{l,l+1}(t)} - 1 \right) \sqrt{\mathcal J^{l,l+1}(t)} }{\sigma_w^2 \log \mathcal J^{l,l+1}(t) } \right\}
\end{align}
guarantees a convergence. In this case, the value of $\mathcal J^{l, l+1}(t)$ can be used to create a scheduler for \Cref{alg:j_train}. 

Moreover, one can solve \eqref{eq:relu_jupdate} and find a learning rate that allows the \Cref{alg:j_train} to converge in 1-step:
\begin{align}\label{eq:1step_lr}
    \eta^l_{\mathrm{1-step}} = \frac{\left( \sqrt{\mathcal J^{l,l+1}(0)} - 1 \right) \sqrt{\mathcal J^{l,l+1}(0)} }{\sigma_w^2 \log \mathcal J^{l,l+1}(0) } \,,
\end{align}

Next we study the dynamics of the optimization while using a single learning rate $\eta$. We estimate the allowed maximum learning rate $\eta_0$ at $t=0$ using $\mathcal J^{l, l+1}(0) = (a_W^{l+1} \sigma_w)^2 / 2$:
\begin{align}\label{eq:eta_0_jl}
    \eta_0 = \frac{\left(a_W^{l+1}\sigma_w - \sqrt{2}\right) a_W^{l+1}}{\sigma_w \left(\log \left[(a_W^{l+1}\sigma_w)^2 \right] - \log 2\right)}  \,.
\end{align}

In \Cref{fig:relu_jac}, we checked our results with \cref{alg:j_train}. All $\mathcal J^{l,l+1}(t)$ values plotted in the figure agree with the values we obtained by iterating \eqref{eq:relu_jupdate} for $t$ steps. The gap between $\eta_0$ and trainable regions can be explained by analyzing \eqref{eq:eta_t}. Assuming at time $t$: $|\mathcal J^{l,l+1}(t) - 1| < |\mathcal J^{l,l+1}(0) - 1|$ holds. For $\mathcal J^{l,l+1} < 1$ if we use a learning rate $\eta$ that satisfies $\eta_0 < \eta < \eta_t$, there is still a chance that \Cref{alg:j_train} can converge. For $\mathcal J^{l,l+1} > 1$ if $\eta_0 > \eta > \eta_t$ holds, \Cref{alg:j_train} may diverge at a later time. A similar analysis for JSL is performed in the appendix.
\begin{figure}[h]
    \centering
    \includegraphics[width=0.75\textwidth]{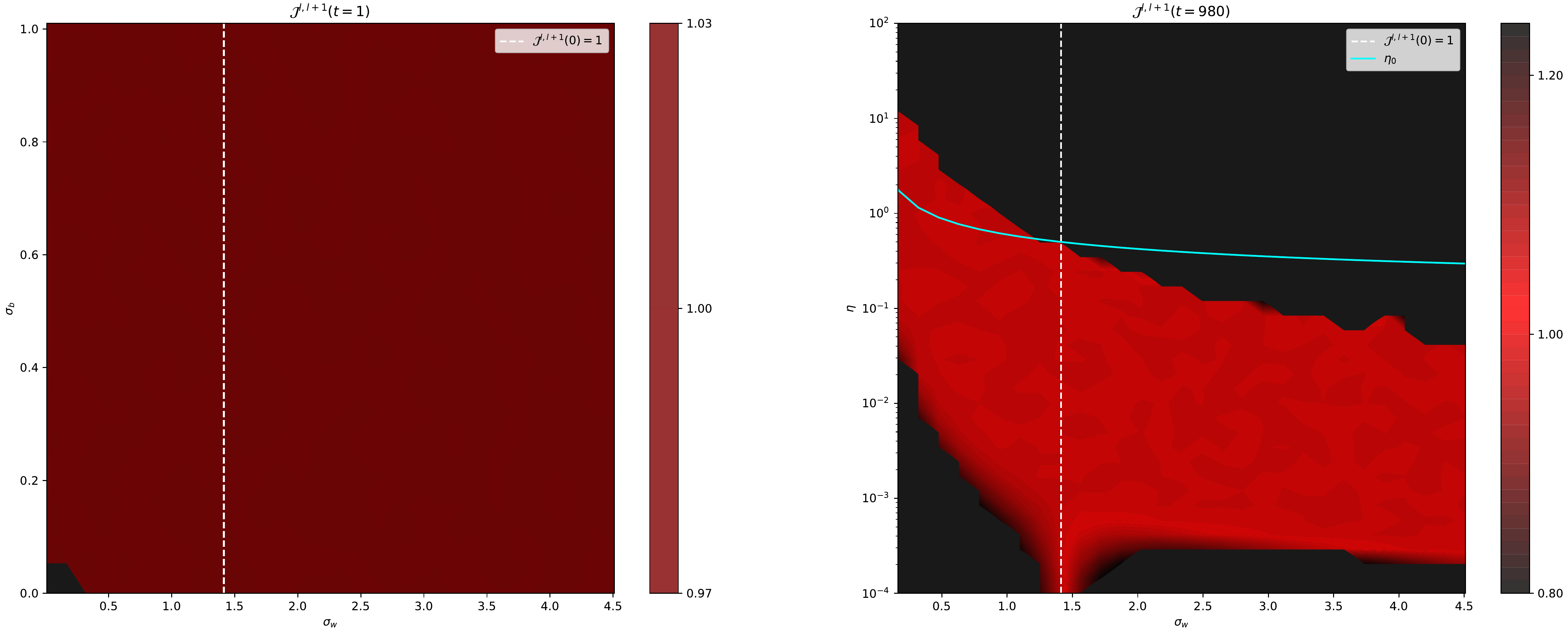}
    \caption{$\mathcal J^{l, l+1}(t)$ plot for $L=10$, $N_l=500$ ReLU MLP networks, initialized with $a_W^l=1$. From left to right: 1) $\mathcal{J}^{l, l+1}(t=1)$ values are obtained by tuning with \Cref{alg:j_train} using $\eta_{\mathrm{1-step}}$ with JLL; 2) we scan in $\eta$-$\sigma_w$ plane using $\sigma_b=0$ networks, tune $\mathcal J^{l, l+1}(t)$ using \Cref{alg:j_train} with JLL for $980$ steps. Only $0.8< \mathcal J^{l,l+1} <1.25$ points are plotted;  All networks are trained with normalized CIFAR-10 dataset, $|B|=256$.}
    \label{fig:relu_jac}
\end{figure}

\section{BatchNorm, Residual Connections and General Strategy}
\label{sec:bn}
\paragraph{BatchNorm and Residual Connections} For MLP networks, the APJN value is only a function of $t$ and it is independent of $|B|$. This property holds except when there is a BatchNorm (BN).

We consider a Pre-BN MLP network with residual connections. The preactivations are given by
\begin{align}\label{eq:bnmlp_preact}
    h^{l+1}_{x; i} = \sum_{j=1}^N a_W^{l+1} W^{l+1}_{ij} \phi(\tilde h^l_{x; j}) + a_b^{l+1} b^{l+1}_i + \mu h^l_{x;i} \,,
\end{align}
where we label different inputs with indices $x,x^\prime, \cdots$ and $\mu$ quantifies the strength of the residual connections (common choice is $\mu=1$). At initialization, the normalized preactivations are defined as
\begin{align}
    \tilde h^l_{x; i} = \frac{h^l_{x; i} - \frac{1}{|B|}\sum_{x' \in B} h^l_{x'; i} }{\sqrt{\frac{1}{|B|} \sum_{x' \in B} \left( h^l_{x'; i} \right)^2 - \left(\frac{1}{|B|} \sum_{x' \in B} h^l_{x'; i}\right)^2  }} \,.
\end{align}
The change in batch statistics leads to non-trivial $\mathcal J^{l,l+1}$ values, which can be approximated using \Cref{conj:bn}.

\begin{conjecture}[APJN with BatchNorm]\label{conj:bn}
    In infinite width limit and at large depth $l$, APJN of Pre-BN MLPs \eqref{eq:bnmlp_preact} converges to a deterministic value determined by the NNGP kernel as $B \rightarrow \infty$:
    \begin{align}
        \mathcal J^{l, l+1} \xrightarrow{|B| \rightarrow \infty} & (a_W^{l+1} \sigma_w)^2 \mathbb E_{\tilde h^l_{x;j} \sim \mathcal N(0, 1)} \left[\phi'(\tilde h^l_{x;j}) \phi'(\tilde h^l_{x;j}) \right] \frac{1}{\mathcal K^l_{xx} - \mathcal K^l_{xx'}} + \mu^2 \,,
    \end{align}
    where the actual value of indices $x'$ and $x$ is not important, as long as $x \neq x'$.
\end{conjecture}
\begin{remark}
    Under the condition of \Cref{conj:bn} $\mathcal J^{l, l+1} \xrightarrow{|B| \rightarrow \infty} 1 + O(l^{-1})$, if $\mu=1$. The finite $|B|$ correction is further suppressed by $l^{-1}$.
\end{remark}
In \Cref{fig:bn_relu} we show numerical results that can justify our conjecture, where we empirically find that finite $|B|$ correction for $|B| \geq 128$. Analytical details are in appendix. Similar results without residual connections have been obtained for finite $|B|$ by \citet{yang2018mean}. 
\begin{figure}[h]
    \centering
    \includegraphics[width=1.0\textwidth]{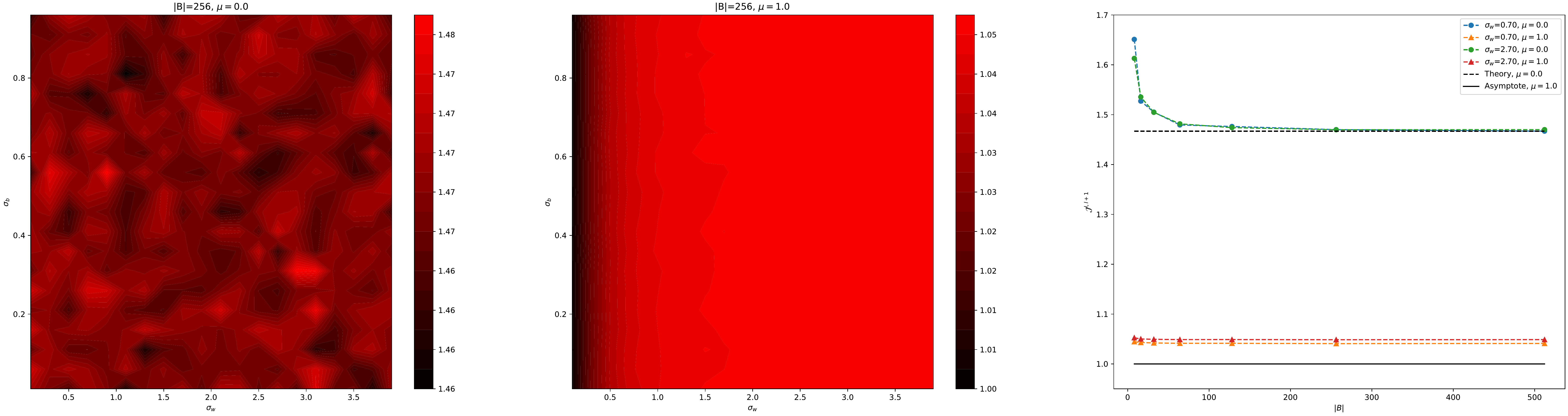}
    \caption{$\mathcal J^{l,l+1}(0)$ phase diagrams for $|B|=256$ in $\sigma_b-\sigma_w$ plane ($\mu=0$ and $\mu=1$); $\mathcal J^{l, l+1}$-$|B|$ plot. From left to right: 1) Pre-BN MLP networks with $\mu=0$ are everywhere chaotic; 2) Pre-BN MLP networks with $\mu=1$ are critical everywhere; 3) For $|B|\geq 128$, the finite $|B|$ corrections are negligible. In all plots we use $L=30$, $N_l=500$, $a_W^l(0)=1$ and averaged over 50 initializations.}
    \label{fig:bn_relu}
\end{figure}

\paragraph{General Strategy} For general network architectures, we propose the following strategy for using \Cref{alg:j_general} with normalized inputs:
\begin{itemize}
    \item If the network does not have BatchNorm, use the algorithm with $|B|=1$.
    \item If the network has BatchNorm, and the user has enough resources, use the algorithm with a $|B|$ which will be used for training. When $|B|$ is large, one should make $\mathcal J^{l,l+1}$ vs. $|B|$ plots like the one in \Cref{fig:bn_relu}, then choose a $|B|$ that needs less computation.
    \item When resources are limited, one can use a non-overlapping set $\{\mathcal{J}^{l, l+k}\}$ with $k>1$ to cover the whole network.
\end{itemize}
The computational cost of the algorithm depends on $k$ and $|B|$.

\section{Experiments}
\label{sec:exp}
In this section, we use a modified version of $\mathcal L_{\log}$, where we further penalize the ratio between NNGP kernels from adjacent layers. The Jacobian-Kernel Loss (JKL) is defined as:
\begin{align}\label{eq:jkle}
    \mathcal L_{\mathcal J \mathcal K\log} = \frac{1}{2} \sum_{l=0}^{L+1} \left[\log(\mathcal J^{l, l+1})\right]^2 + \frac{\lambda}{2} \sum_{l=0}^{L+1} \left[\log\left (\frac{\mathcal K^{l+1}(x, x)}{\mathcal K^l(x,x)} \right)\right]^2 \,,
\end{align}
where we introduced an extra hyperparameter $\lambda$ to control the penalization strength. We also included input and output layers. Both APJNs and NNGP kernels will be calculated using flattened preactivations.

\subsection{ResMLP}

ResMLP \citep{touvron2021resmlp} is an architecture for image recognition built entirely on MLPs. It offers competitive performance in both image recognition and machine translation tasks. The architecture consists of cross-channel and cross patch MLP layers, combined with residual connections. The presence of residual connections and the absence of normalization techniques such as LayerNorm \citep{ba2016layer} or BatchNorm \citep{ioffe2015batch} render ResMLP to be initialized off criticality.
To mitigate this issue, ResMLP architecture utilizes LayerScale \citep{touvron2021cait}; which multiplies the output residual branch with a trainable matrix, initialized with small diagonal entries.

\paragraph{CIFAR-10} Here we obtain a critical initialization for ResMLP-S12 using \Cref{alg:j_train} with loss \eqref{eq:jkle}, with $a^l_{\theta}$ introduced for all layers. In our initialization, the ``smallnes'' is distributed across all parameters of the residual block, including those of linear, affine normalization and LayerScale layers. As we show in \Cref{fig:resmlp}, Kaiming initialization is far from criticality. $\texttt{AutoInit}$ finds an initialization with almost identical $\{\mathcal J^{l, l+1} \}$ and similar $\{\mathcal K^{l, l+1}(x, x)\}$ compared to the prescription proposed by \citet{touvron2021resmlp}. 
\begin{figure}[h]
    \centering
    \includegraphics[width=\textwidth]{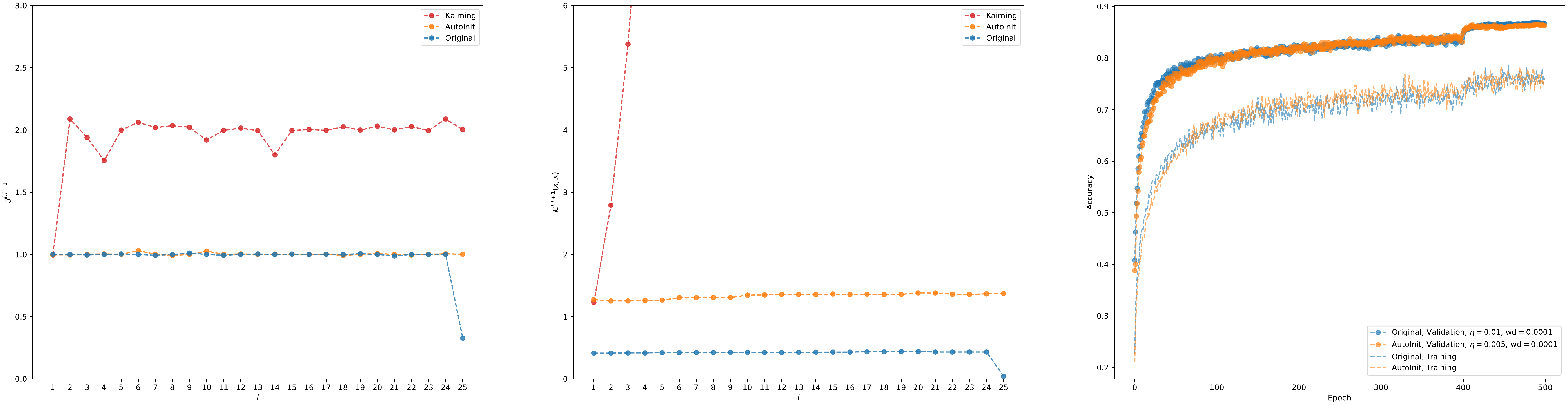}
    \caption{From left to right: 1) and 2) Comparing $\mathcal J^{l, l+1}$ and $\mathcal K^l(x,x)$ for ResMLP-S12 for Kaiming, original and $\texttt{AutoInit}$ initializations. Depth $l$ is equal to the number of residual connections. The network function in the $\texttt{AutoInit}$ case is very close to identity at initialization. 3) Training and validation accuracy. Both, original and \texttt{AutoInit}, models are trained on CIFAR-10 dataset for 600 epochs using \texttt{LAMB} optimizer\citep{You2020Large} with $|B|=256$. The learning rate is decreased by a factor of 0.1 at 450 and 550 epochs. Training accuracy is measured on training samples with Mixup $\alpha=0.8$. Both models interpolate the original training set.}
    \label{fig:resmlp}
\end{figure}

\paragraph{ImageNet \citep{liILSVRC15}} We report $74.0\%$ top-1 accuracy for ResMLP-S12 initialized using \texttt{AutoInit}, whereas the top-1 accuracy reported in \citep{touvron2021resmlp} for the same architecture is $76.6\%$. The model has $15$ million parameters. We used a setup similar to the one in original paper, which is based on timm library \citep{rw2019timm} under Apache-2.0 license \citep{apachev2}. However, we made the following modifications in our training: 1) We use learning rate $\eta=0.001$ and $|B|=1024$. 2) We use mixed precision. 3) We do not use \texttt{ExponentialMovingAverage}. The training was performed on two NVIDIA RTX 3090 GPUs; and took around $3.5$ days to converge (400 epochs).

The auto-initialized model are obtained by tuning the Kaiming initialization using \Cref{alg:j_train} with $\mathcal L_{\mathcal J \mathcal K\log}(\lambda=0.5)$, $\eta=0.03$ and $|B|=32$ for 500 steps.
\subsection{VGG}
VGG \citep{simonyan2014very} is an old SOTA architecture, which was notoriously difficult to train before Kaiming initialization was invented. The BatchNorm variants $\mathrm{VGG19\_BN}$ further improve the training speed and performances compared to the original version.  PyTorch version of VGG \citep{NEURIPS2019_9015} is initialized with $\mathrm{fan\_out}$ Kaiming initialization \citep{he2015delving}. In \Cref{fig:bn_relu} we show that the BatchNorm makes Kaiming-initialized ReLU networks chaotic.

We obtain a close to critical initialization using \Cref{alg:j_train} for $\mathrm{VGG19\_BN}$, where we introduce the auxiliary parameters $a^l_{\theta}$ for all BatchNorm layers. $\mathcal J^{l, l+1}$ is measured by the number of composite (Conv2d-BatchNorm-ReLU) blocks or MaxPool2d layers.

We compared $\mathcal J^{l, l+1}$, $\mathcal K^l(x,x)$ and accuracies on CIFAR-10 datasets between auto-initialized model and the one from PyTorch\citep{krizhevsky2009learning}, see \Cref{fig:vgg}. 
\begin{figure}[h]
    \centering
    \includegraphics[width=\textwidth]{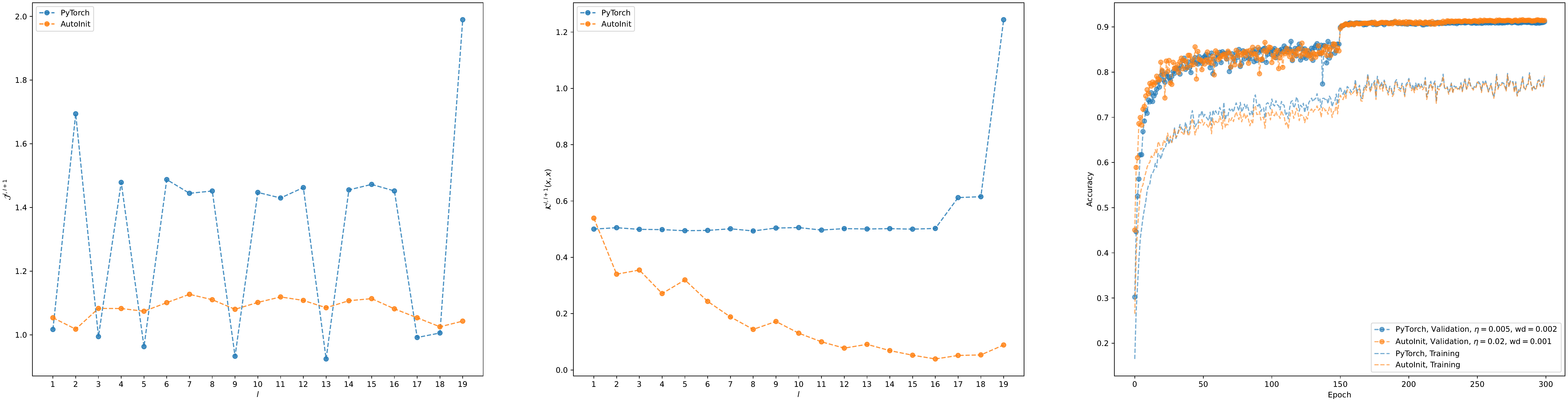}
    \caption{From left to right: 1) and 2) comparing $\mathcal J^{l, l+1}$ and $\mathcal K^l(x,x)$ between PyTorch version $\mathrm{VGG19\_BN}$ and \texttt{AutoInit} version, we ensure $\mathcal J^{l, l+1}=1$ with a high priority ($\lambda=0.05$); 3) training and validation accuracy. We train both models on CIFAR-10 dataset using SGD with $\mathrm{momentum}=0.9$ and $|B|=256$ for 300 epochs, where we decrease the learning rate by a factor of 0.1 at 150 and 225 epochs. Training accuracy is measured on training samples with mixup $\alpha=0.8$. Both models interpolate the original training set.}
    \label{fig:vgg}
\end{figure}

\section{Conclusions}
\label{sec:conclu}
In this work we have introduced an algorithm, \texttt{AutoInit}, that allows to initialize an arbitrary feed-forward deep neural network to criticality. \texttt{AutoInit} is an unsupervised learning algorithm that forces norms of all nearby partial Jacobians to have a unit norm via minimizing the loss function \eqref{eq_loss_sq_J}. A slight variation of the \texttt{AutoInit} also tunes the forward pass to ensure that gradients in all layers of a DNN are well-behaved.

To gain some intuition about the algorithm we have solved the training dynamics for MLPs with ReLU activation and discussed the choice of hyperparameters for the tuning procedure that ensures its convergence.

Then we have evaluated the performance of \texttt{AutoInit}-initialized networks against initialization schemes used in literature. We considered two examples: ResMLP architecture and VGG. The latter was notoriously difficult to train at the time it was introduced. \texttt{AutoInit} finds a good initialization (somewhat close to Kaiming) and ensures training. ResMLP uses a variation of ReZero initialization scheme that puts it close to dynamical isometry condition. \texttt{AutoInit} finds a good initialization that appears very different from the original, however the network function is also very close to the identity map at initialization. In both cases the performance of the \texttt{AutoInit}-initialized networks is competitive with the original models. We emphasize that \texttt{AutoInit} removes the necessity for trial-and-error search for a working initialization.

We expect that \texttt{AutoInit} will be useful in automatic neural architecture search tasks as well as for general exploration of new architectures.

\begin{ack}
T.H., D.D. and A.G. were supported, in part, by the NSF CAREER Award DMR-2045181 and by the Salomon Award.
\end{ack}
\bibliographystyle{plainnat}
\bibliography{refs}


\newpage
\include{SM}

\end{document}

%% file: SM.tex
\appendix
\section{Experimental Details}
\Cref{fig:relu_jac}: The the second panel is made of 1200 points, each point takes around $1.5$ minutes running on a single single NVIDIA RTX 3090 GPU.

\Cref{fig:bn_relu}: We scanned over $400$ points for each phase diagram, which overall takes around $5$ hours on a single NVIDIA RTX 3090 GPU.

\Cref{fig:resmlp}: We use \cref{alg:j_train} to tune our model on CIFAR-10 dataset. We use SGD with $\eta=0.03$ and $N_v=2$ for $392$ steps, $|B|=256$. The training curves we reported are selected from the best combination from the following hyperparameters $\mathrm{lr}=\{0.005, 0.01\}$, $\mathrm{weight\; decay}=\{10^{-5}, 10^{-4}\}$. We used RandAugment\citep{cubuk2020RandAug}, horizontal flip, Mixup with $\alpha=0.8$ \citep{zhang2018mixup} and Repeated-augmentation \citep{hoffer2020aug}. All of our results are obtained using a single NVIDIA RTX 3090 GPU.

\Cref{fig:vgg}: We use \cref{alg:j_train} to tune our model on CIFAR-10 dataset for $392$ steps with $\eta=0.01$  $|B|=128$ and $N_v=3$. The training curves we reported are selected from the best combination from the following hyperparameters $\mathrm{lr}=\{0.001, 0.002, 0.005, 0.01, 0.02\}$, $\mathrm{weight\; decay}=\{0.0005, 0.001, 0.002, 0.005 \}$. We used RandAugment\citep{cubuk2020RandAug}, horizontal flip and Mixup with $\alpha=0.8$ \citep{zhang2018mixup} and Repeated-augmentation \citep{hoffer2020aug}. We froze auxiliary parameters instead of scale the weights. All of our results are obtained using a single NVIDIA RTX 3090 GPU.

\Cref{figapp:relu_jac}: Exactly the same as \Cref{fig:relu_jac}, except we used JSL.

\section{Theoretical Details}
\subsection{Factorization of APJN}
We the factorization property using MLP networks in infinite width limit. This proof works for any iid $\theta^l$ where $|\theta^l|$ has some $2+\delta$ moments.

We start from the definition of the partial Jacobian, set $a^l_{\theta}=1$ for simplicity.
\begin{align}
    \nonumber
    \mathcal{J}^{l,l+2} &\equiv \frac{1}{N_{l+2}} \mathbb E_{\theta} \left[ \sum_{i=1}^{N_{l+2}}\sum_{j=1}^{N_{l}} \frac{\partial h^{l+2}_i}{\partial h^{l}_j} \frac{\partial h^{l+2}_i}{\partial h^{l}_j} \right ] \\
    \nonumber
    &= \frac{1}{N_{l+2}} \mathbb E_{\theta} \left[ \sum_{i=1}^{N_{l+2}}\sum_{j=1}^{N_l} \sum_{k,m=1}^{N_{l+1}} 
    \left(W^{l+2}_{ik} \phi'(h^{l+1}_k) \right)
    \left(W^{l+2}_{im} \phi'(h^{l+1}_m) \right)
    \left(\frac{\partial h^{l+1}_k}{\partial h^{l+1}_j} \frac{\partial h^{l+1}_m}{\partial h^{l}_j}\right)
    \right ] \\
    \nonumber
    &= \frac{\sigma_w^2}{N_{l+2} N_{l+1}} \sum_{i=1}^{N_{l+2}}\sum_{j=1}^{N_{l}} \sum_{k=1}^{N_{l+1}} \mathbb E_{\theta} \left[  \phi'(h^{l+1}_k) \phi'(h^{l+1}_k)
    \frac{\partial h^{l+1}_k}{\partial h^{l}_j}
    \frac{\partial h^{l+1}_k}{\partial h^{l}_j} \right ] \\
    \nonumber
    &= \frac{\sigma_w^2}{N_{l+2} N_{l+1}} \sum_{i=1}^{N_{l+2}}\sum_{j=1}^{N_{l}} \sum_{k=1}^{N_{l+1}} \mathbb E_{W^{l+1}, b^{l+1}} \left[ \phi'(h^{l+1}_k) \phi'(h^{l+1}_k)
    W^{l+1}_{kj}  W^{l+1}_{kj} \right] \mathbb E_{\theta} \left[\phi'(h^l_k) \phi'(h^l_k) \right ] \\
    &= \mathcal J^{l,l+1} \mathcal J^{l+1, l+2} + O\left(\frac{\chi^l_{\Delta}}{N_l} \right) \,,
\end{align}
where the $1/N_l$ correction is zero in the infinite width limit. We used the fact that in infinite width limit $h^{l+1}_k$ is independent of $h^l_k$, and calculated the first expectation value of the fourth line using integration by parts. Recall that for a single input (omit x)
\begin{align}
    \chi^{l}_{\Delta} \equiv (a_W^{l+1} \sigma_w)^2 \mathbb E_{\theta} \left[\phi''(h^{l}_i) \phi''(h^{l}_i) + \phi'(h^{l}_i) \phi'''(h^{l}_i) \right] \,.
\end{align}

\subsection{Exploding and Vanishing Gradients}
We show details for deriving \eqref{eq:tauto_aw} for MLP networks, assuming $l'>l$:
\begin{align}
    \frac{\partial \log \mathcal J^{l',l'+1}}{\partial a_W^{l+1}} =& \frac{1}{\mathcal J^{l',l'+1}}\frac{\partial \mathbb E_{h^{l'}_i \sim \mathcal N(0, \mathcal K^{l'}(x,x))} \left[(a_W^{l'+1} \sigma_w)^2 \phi'(h^{l'}_i) \phi'(h^{l'}_i) \right]}{\partial a_W^{l+1}} \nonumber \\
    =& \frac{1}{\mathcal J^{l',l'+1}}\frac{\partial}{\partial K^{l'}(x,x)} \left(\frac{{(a_W^{l'+1} \sigma_w)^2}}{\sqrt{2\pi \mathcal K^{l'}(x,x)}} \int \phi'(h^{l'}_i) \phi'(h^{l'}_i) e^{-\frac{h^{l'}_i h^{l'}_i}{2\mathcal K^{l'}(x,x)}} dh^{l'}_i \right) \frac{\partial K^l(x,x')}{\partial a_W^{l+1}}\nonumber \\
    =& \frac{2}{\mathcal J^{l',l'+1}} \chi^{l'}_{\Delta} \frac{\partial \mathcal K^{l'}(x, x)}{\partial \mathcal K^{l'-1}(x,x)} \cdots \frac{\partial \mathcal K^{l+1}(x,x)}{\partial a_W^{l+1}} \nonumber \\
    =& \frac{4}{a_W^{l+1} \mathcal J^{l',l'+1}} \chi^{l'}_{\Delta} \chi^{l'-1}_{\mathcal K} \cdots \chi^{l+1}_{\mathcal K} \mathcal K^{l+1}(x,x) \,,
\end{align}
where we calculated the derivative respect to $\mathcal K^{l'}(x,x)$, then used integration by parts to get the third line. The derivation for \eqref{eq:tauto_ab} is similar.

\subsection{ReLU Details}
\paragraph{Learning rate $\eta$} The learning rate bound \eqref{eq:eta_t} is obtained by requiring $|\sqrt{\mathcal J^{l,l+1}(t)} -1 |$ to decrease monotonically with time $t$.

\paragraph{\texorpdfstring{Derivation for $\chi_{\Delta}^l=0$}{}} This is straightforward to show by direct calculation in the infinite width limit. We set $a_W^l=1$ and ignore neuron index $i$ for simplicity. 
\begin{align}
    \chi^{l'}_{\Delta} =& \sigma_w^2 \mathbb E_{h^l \sim \mathcal N(0, \mathcal K^l(x,x))} \left[\phi''(h^{l}) \phi''(h^{l}) + \phi'(h^{l}) \phi'''(h^{l'}) \right] \nonumber \\
    &= \sigma_w^2 \mathbb E_{h^l \sim \mathcal N(0, \mathcal K^l(x,x))} \left [ \frac{d}{dh^{l}} \left( \phi'(h^l) \phi''(h^l)  \right) \right ] \nonumber \\
    &= \sigma_w^2 \mathbb E_{h^l \sim \mathcal N(0, \mathcal K^l(x,x))} \left [ \frac{d}{dh^{l}} \left( \Theta(h^l) \delta(h^l)  \right) \right ] \nonumber \\
    &= \sigma_w^2 \mathbb E_{h^l \sim \mathcal N(0, \mathcal K^l(x,x))} \left [ \frac{h^l}{\mathcal K^l(x,x)} \left( \Theta(h^l) \delta(h^l)  \right) \right ] \nonumber \\
    &= 0 \,,
\end{align}
where $\Theta(h^l)$ is Heaviside step function and $\delta(h^l)$ is Dirac delta function. To get the last line we used $h^l \delta(h^l)=0$.

\subsection{\texorpdfstring{\Cref{conj:bn}}{}}
Here we offer an non-rigorous explanation for the conjecture in the infinite $|B|$ and the infinite width limit. We use a MLP model with $a_{\theta}^l=1$ as an example.

We consider
\begin{align}
    h^{l+1}_{x; i} = \sum_{j=1}^N W^{l+1}_{ij} \phi(\tilde h^l_{x; j}) + b^{l+1}_j + \mu h^l_{x; j}\,,
\end{align}
where 
\begin{align}\label{eq:BN}
    \tilde h^l_{x; i} =& \frac{h^l_{x; i} - \frac{1}{|B|}\sum_{x' \in B} h^l_{x'; i} }{\sqrt{\frac{1}{|B|} \sum_{x' \in B} \left( h^l_{x'; i} \right)^2 - \left(\frac{1}{|B|} \sum_{x' \in B} h^l_{x'; i}\right)^2  }} \nonumber \\
    =& \frac{\sqrt{|B|} \sum_{x' \in B} P_{x x'} h^l_{x'; i}}{\sqrt{ \sum_{x \in B} \left(\sum_{x' \in B} P_{x x'} h^l_{x'; i}\right)^2 }} \,,
\end{align}
where $P_{xx'} \equiv \delta_{xx'} - 1/ |B|$. It is a projector in the sense that $\sum_{x' \in B} P_{xx'} P_{x'x''} = P_{xx''}$.

Derivative of the normalized preactivation:
\begin{align}
    \frac{\partial \tilde h^l_{x; i}}{\partial h^l_{x';j}} = \sqrt{|B|} \left(\frac{P_{xx'}}{\sqrt{\sum_{x \in B} \left(\sum_{x'' \in B} P_{x,x''} h^l_{x'';i}\right)^2}} - \frac{\sum_{x'' \in B} P_{xx''} h^l_{x''; i} \sum_{x'' \in B} P_{x'x''} h^l_{x''; i}}{\left( \sqrt{\sum_{x\in B} \left(\sum_{x'' \in B} P_{x,x''} h^l_{x'';i}\right)^2 } \right)^3} \right)\delta_{ij} \,.
\end{align}

Then the one layer APJN:
\begin{align}\label{eqapp:bn_apjn}
    \mathcal J^{l, l+1} =& \frac{\sigma_w^2}{N_l} \sum_{x,x' \in B} \sum_{j=1}^{N_l} \mathbb E_{\theta} \left[\left(\phi'(\tilde h^l_{x; j}) \right)^2 \left(\frac{P_{xx'}}{\sqrt{\sum_{x \in B} \left(\sum_{x'' \in B} P_{x,x''} h^l_{x''; j}\right)^2}} \right. \right. \nonumber \\
    &\left. \left.- \frac{\sum_{x'' \in B} P_{xx''} h^l_{x''; j} \sum_{x'' \in B} P_{x'x''} h^l_{x''; j}}{\left( \sqrt{\sum_{x \in B} \left(\sum_{x'' \in B} P_{x,x''} h^l_{x''; j}\right)^2 } \right)^3} \right)^2 \right] + \mu^2 \,.
\end{align}

In the infinite $|B|$ limit, only one term can contribute:
\begin{align}
    \mathcal J^{l, l+1} =& \frac{\sigma_w^2}{N_l} \mathbb E_{\theta} \left[\sum_{x,x' \in B} \sum_{j=1}^{N_l} \left(\phi'(\tilde h^l_{x; j}) \right)^2 \frac{P_{xx'} P_{xx'}}{\sum_{x=1}^B \left(\sum_{x'' \in B} P_{x x''} h^l_{x''; j}\right)^2}\right]  + \mu^2 + O\left(\frac{1}{|B|}\right) \nonumber \\
    =& \frac{\sigma_w^2}{N_l} \mathbb E_{\theta} \left[ \sum_{j=1}^{N_l} \sum_{x \in B} \left(\phi'(\tilde h^l_{x; j}) \right)^2 \frac{P_{xx}}{\sum_{x \in B} \left(\sum_{x'' \in B} P_{x x''} h^l_{x''; j}\right)^2} \right] + \mu^2 + O\left(\frac{1}{|B|}\right) \nonumber \\
    =& \frac{\sigma_w^2}{N_l} \mathbb E_{\theta} \left[\sum_{j=1}^{N_l} \left( \left[\frac{1}{|B|} \sum_{x \in B} \left(\phi'(\tilde h^l_{x; j}) \right)^2 \right] \frac{|B|-1}{\sum_{x \in B} \left(\sum_{x'' \in B} P_{x x''} h^l_{x''; j}\right)^2} \right) \right] + \mu^2 + O\left(\frac{1}{|B|}\right) \nonumber \\
    \xrightarrow{B \rightarrow \infty} & \frac{\sigma_w^2}{N_l} \sum_{j=1}^{N_l} \mathbb E_{\tilde h^l_{x;j} \sim \mathcal N(0,  \delta_{xx'} )} \left[\phi'(\tilde h^l_{x;j}) \phi'(\tilde h^l_{x;j}) \right] \frac{1}{\mathcal K^l_{xx} - \mathcal K^l_{xx'}} + \mu^2 \,,
\end{align}
where $x'$ is a dummy index, just to label the off-diagonal term. We used \cref{conjecture:proj} and \cref{conjecture:tilde_h} to get the result.

\begin{conjecture}[Projected Norm]\label{conjecture:proj}
    In the infinite width limit. For a large depth $l$, $\frac{1}{|B|} \sum_{\hat{x} \in B} \left(\sum_{x' \in B} P_{\hat{x}x'} h^l_{x';j}\right)^2$ converges to a deterministic value $\frac{|B|-1}{|B|} \left(\mathcal K^l_{x'x'} - \mathcal K^l_{x' x''} \right)$ as batch size $|B| \rightarrow \infty$.
\end{conjecture}
\begin{proof}[Non-regirous "proof"]
    In the infinite width limit $h^l_{x; j}$ is sampled from a Gaussian distribution $\mathcal N(0, \mathcal K^l_{xx'})$, where the value $\mathcal K_{xx'}$ only depends on if $x$ is the same as $x'$ or not. 

    We first simplify the formula:
    \begin{align}\label{eq:proj}
        &\frac{1}{|B|} \sum_{\hat{x} \in B} \left(\sum_{x' \in B} P_{\hat{x}x'} h^l_{x';j}\right)^2  \nonumber \\
        =& \frac{1}{|B|} \sum_{x', x'' \in B} P_{x' x''} h^l_{x'; j} h^l_{x''; j} \nonumber \\
        =& \frac{1}{|B|} \left(\sum_{x' \in B} (h^l_{x'; j})^2 - \frac{1}{|B|} \sum_{x', x'' \in B} h^l_{x';j} h^l_{x'';j} \right) \nonumber \\
        =& \frac{1}{|B|} \left(\frac{|B|-1}{|B|} \sum_{x' \in B} (h^l_{x'; j})^2 - \frac{1}{|B|} \sum_{x' \neq x''}^B h^l_{x';j} h^l_{x'';j} \right) \,.
    \end{align}    
    The average over $x'$ and $x''$ in infinite $|B|$ limit can be replaced by integration over their distribution (this is the non-rigorous step, complete rigorous proof see \citet{yang2018mean}):
    \begin{align}
        &\frac{1}{|B|} \sum_{\hat{x} \in B} \left(\sum_{x' \in B} P_{\hat{x}x'} h^l_{x';j}\right)^2 \nonumber \\
        \xrightarrow{|B| \rightarrow \infty} & \frac{|B|-1}{|B|}  \left(\mathbb E_{h^l_{x';j} \sim \mathcal N(0, \mathcal K^l_{xx'}) } \left[(h^l_{x'; j})^2 \right] - \mathbb E_{h^l_{x';j} \sim \mathcal N(0, \mathcal K^l_{x'x''}) } \left[h^l_{x';j} h^l_{x'';j} \right] \right) \nonumber \\
        = & \frac{|B|-1}{|B|} \left(\mathcal K^l_{x'x'} - \mathcal K^l_{x' x''} \right) \,,
    \end{align}
\end{proof}

Next, we need to show how to calculate $\mathcal K^{l}_{xx'}$. Before that we first try to simplify find the distribution of $\tilde h^l_{x; i}$ in the infinite $|B|$ limit. 

\begin{conjecture}[$\tilde h^l_{x;i}$ distribution]\label{conjecture:tilde_h}
    In the infinite $|B|$ limit and the infinite width limit, assume for large depth $\mathcal K^l_{xx}$ reaches a fixed point. Then $\tilde h^l_{x;i}$ can be seen as sampled from a Gaussian distribution with the covariance matrix
    \begin{align}
        \lim_{|B| \rightarrow \infty} \mathbb E_{\theta} \left[\tilde h^l_{x;i} \tilde h^l_{y;j}\right] = & \mathbb E_{\theta} \left[ \frac{\sum_{x', x'' \in B} P_{x x'} P_{y x''} h^l_{x';i} h^l_{x'';j}}{\frac{|B|-1}{|B|} \left(\mathcal K^l_{xx} - \mathcal K^l_{x \hat{x}}\right)} \right] \nonumber \\
        =& \frac{\sum_{x', x'' \in B} P_{xx'} P_{yx''} \mathcal K^l_{xx'}}{\frac{|B|-1}{|B|} \left(\mathcal K^l_{xx} - \mathcal K^l_{x \hat{x}}\right)} \delta_{ij} \nonumber \\
        =& \frac{\mathcal K^l_{xy} - \frac{1}{|B|} \mathcal K^l_{\hat{x}\hat{x}} - \frac{|B|-1}{|B|} \mathcal K^l_{x\hat{x}}}{\frac{|B|-1}{|B|} \left(\mathcal K^l_{xx} - \mathcal K^l_{x \hat{x}}\right)} \delta_{ij} \nonumber \\
        =&\delta_{xy} \delta_{ij} \,\,
    \end{align}
    where we used \cref{conjecture:proj} in the first line.
\end{conjecture}

For ReLU:
\begin{align}
    \mathcal K^{l+1}_{xx'} = \begin{cases}
        \frac{\sigma_w^2}{2} + \sigma_b^2 + \mu^2 \mathcal K^l_{xx} & \text{if $x=x'$} \\
        \frac{\sigma_w^2}{2\pi} + \sigma_b^2 + \mu^2 \mathcal K^l_{xx'} & \text{if $x \neq x'$} \,.
    \end{cases}
\end{align}
Then for $\mu=0$ APJN is independent of $\sigma_w^2$ and $\sigma_b^2$ in infinite $|B|$ limit:
\begin{align}
    \mathcal J^{l, l+1} = \frac{\pi}{\pi - 1} \,,
\end{align}
and for $\mu=1$:
\begin{align}
    \mathcal J^{l, l+1} = 1 + O \left(\frac{1}{l} \right) \,.
\end{align}
It is also intuitively clear by realizing the denominator of \eqref{eqapp:bn_apjn} is growing with $l$ when $\mu=1$. Thus the finite $|B|$ corrections are further suppressed. We checked our results in \Cref{fig:bn_relu}.

\section{JSL and JKL}
\subsection{JSL}
Since we already discussed our results for JL, we show details for JSL in this section. Derivation for JL is almost identical.

Using JSL, The SGD update of $a^l_{\theta}$ at time t is
\begin{align}\label{eqapp:a_update_jsl}
    a_{\theta}^{l+1}(t+1) - a_{\theta}^{l+1}(t) = - \eta \sum_{l' \geq l}^L  \frac{\partial \mathcal J^{l', l'+1}(t)}{\partial a_{\theta}^{l+1}(t)} \left(\mathcal J^{l', l'+1}(t) - 1 \right) \,.
\end{align}
We focus on ReLU networks to demonstrate the difference between JL and JSL. 

For ReLU networks, we can rewrite \eqref{eqapp:a_update_jsl} as
\begin{align}\label{eqapp:j_update_jsl}
    \sqrt{\mathcal J^{l,l+1}(t+1)} - \sqrt{\mathcal J^{l,l+1}(t)} = -\eta \sigma_w^2 \sqrt{\mathcal J^{l,l+1}(t)} \left(\mathcal J^{l,l+1}(t) - 1\right) \,.
\end{align}

\paragraph{Learning Rate $\eta$} The learning rate limit $\eta_t$ is obtained by requiring $|\sqrt{\mathcal J^{l,l+1}(t)} - 1|$ monotonically decrease with time $t$, for any $l$, then we have
\begin{align}\label{eqapp:eta_t_jsl}
    \eta_t < \min_{1 \leq l \leq L} \left\{\frac{2}{\sigma_w^2 \sqrt{\mathcal J^{l,l+1}(t)} \left(1 + \sqrt{\mathcal J^{l, l+1}(t)}\right)} \right \} \,.
\end{align}

Or by solving \eqref{eqapp:j_update_jsl} with $J^{l,l+1}(1)=1$:
\begin{align}
    \eta_{\mathrm{1-step}} = \frac{1}{\sigma_w^2 \sqrt{\mathcal J^{l,l+1}(0)} \left(1 + \sqrt{\mathcal J^{l, l+1}(0)}\right)}  \,.
\end{align}

For the dynamics of the optimization while using a single learning rate $\eta$. We again estimate the allowed maximum learning rate $\eta_0$ at $t=0$ using $\mathcal J^{l, l+1}(0) = (a_W^{l+1} \sigma_w)^2 / 2$:
\begin{align}
    \eta_0 = \frac{4}{\sigma_w^3 a_W^l \left(\sqrt{2} + a_W^l \sigma_w \right)} \,.
\end{align}

Compared to \eqref{eq:eta_0_jl}, which scales as $1/\log \sigma_w$ for large $\sigma_w$, $\eta_0$ for JSL scales as $\sigma_w^{-4}$ for large $\sigma_w$. This makes the JSL a way worse choice than JLE when $\mathcal J^{l,l+1} \gg 1$.

In \Cref{figapp:relu_jac}, we checked our results with \cref{alg:j_train} using JSL. All other details are the same as \Cref{fig:relu_jac}. The gap between $\eta_0$ and trainable regions can again be explained similarly by analyzing \eqref{eqapp:eta_t_jsl}. Assuming at time $t$: $|\mathcal J^{l,l+1}(t) - 1| < |\mathcal J^{l,l+1}(0) - 1|$ holds. For $\mathcal J^{l,l+1} < 1$ if we use a learning rate $\eta$ that satisfies $\eta_0 > \eta > \eta_t$, there is still a chance that \Cref{alg:j_train} diverges for some $t>0$. For $\mathcal J^{l,l+1} > 1$ if $\eta_0 < \eta < \eta_t$ holds, \Cref{alg:j_train} may say still have a chance to converge for some $t>0$. 
\begin{figure}[h]
    \centering
    \includegraphics[width=0.75\textwidth]{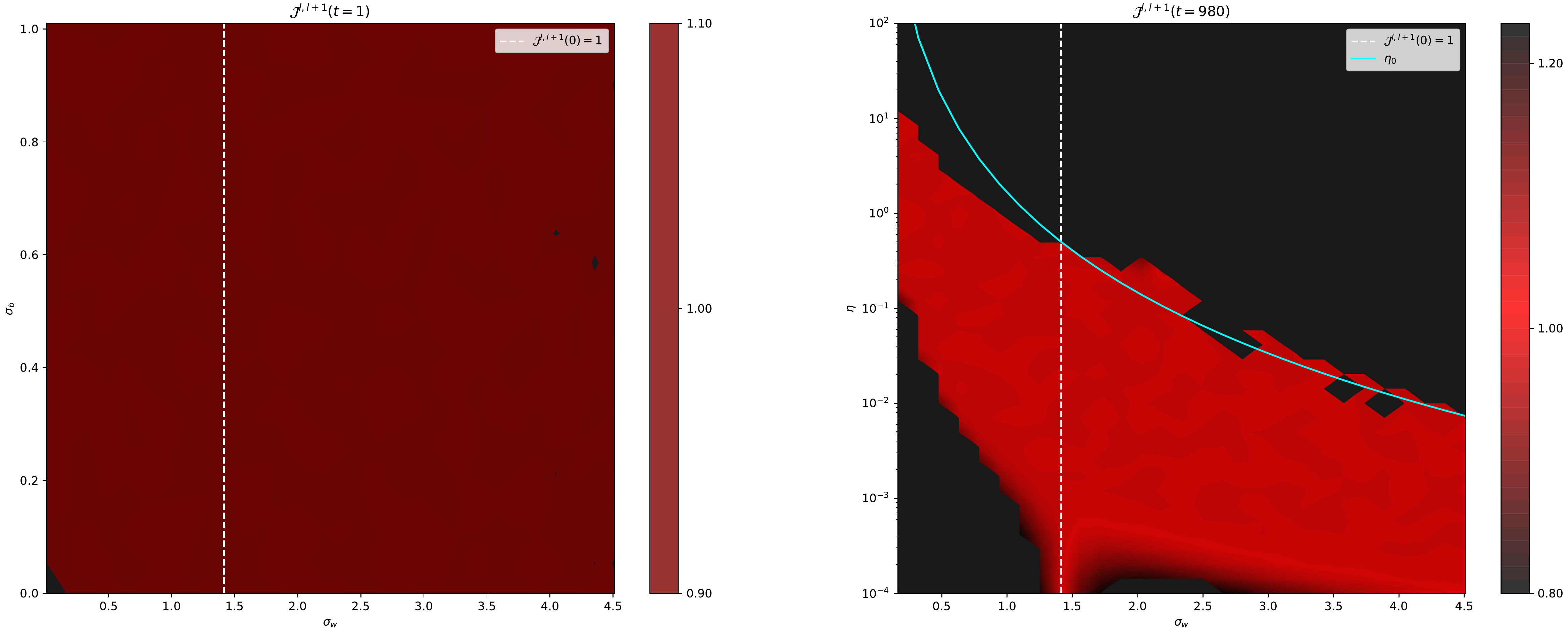}
    \caption{$\mathcal J^{l, l+1}(t)$ plot for $L=10$, $N_l=500$ ReLU MLP networks, initialized with $a_W^l=1$. From left to right: 1) $\mathcal{J}^{l, l+1}(t=1)$ values are obtained by tuning with \Cref{alg:j_train} using $\eta_{\mathrm{1-step}}$ with JSL; 2) we scan in $\eta$-$\sigma_w$ plane using $\sigma_b=0$ networks, tune $\mathcal J^{l, l+1}(t)$ using \Cref{alg:j_train} with JSL for $980$ steps. Only $0.8< \mathcal J^{l,l+1} <1.25$ points are plotted;  All networks are trained with normalized CIFAR-10 dataset, $|B|=256$.}
    \label{figapp:relu_jac}
\end{figure}
\subsection{JKL}
We mentioned the following loss function in the main text.
\begin{align}\label{eqapp:jkle}
    \mathcal L_{\mathcal J \mathcal K\log} = \frac{1}{2} \sum_{l=1}^{L} \left[\log(\mathcal J^{l, l+1})\right]^2 + \frac{\lambda}{2} \sum_{l=1}^{L} \left[\log\left (\frac{\mathcal K^{l+1}(x, x)}{\mathcal K^l(x,x)} \right)\right]^2 \,.
\end{align}
There are other possible choices for controlling the forward pass, we will discuss this one briefly.

First we calculate the derivative from kernel terms. We omit $x$ and $t$ dependency and introduce $r^{l+1, l} = \mathcal K^{l+1} / \mathcal K^l$ for clarity:
\begin{align}
    &\frac{\partial}{\partial a_W^{l+1}}\left( \frac{\lambda}{2} \sum_{l'\geq l}^{L} \left[\log{r^{l'+1,l'}}\right]^2 \right) \nonumber \\
    =& \frac{2\lambda}{a_W^{l+1}} \log r^{l+1,l} + \frac{\partial}{\partial a_W^{l+1}}\left( \frac{\lambda}{2} \sum_{l' > l}^{L} \left[\log{r^{l'+1,l'}}\right]^2 \right) \,,
\end{align}
which has a similar structure as APJN terms. Next we pick a term with $l' > l$ in the parentheses:
\begin{align}
    &\frac{\partial}{\partial a_W^{l+1}}\left( \frac{\lambda}{2} \left[\log{r^{l'+1,l'}}\right]^2 \right) \nonumber \\
    =& \frac{\lambda}{\mathcal K^{l'+1} \mathcal K^{l'}} \left(\mathcal K^{l'} \frac{\partial \mathcal K^{l'+1}}{\partial a_W^{l+1}} - \mathcal K^{l'+1} \frac{\partial \mathcal K^{l'}}{\partial a_W^{l+1}}\right)\log{r^{l'+1,l'}} \,,
\end{align}
which is independent of depth for $\sigma_b=0$, and is always finite for $\sigma_b$.

We find that update from the forward pass term for $a_b^{l+1}$ is subtle. For $\sigma_b=0$, similar to the discussion of APJN terms, the update of $a_b^{l+1}$ is zero. For $\sigma_b > 0$, there are two possibilities:
\begin{itemize}
    \item Unbounded activation functions, when $\chi_{\mathcal K}^l > 1$: $\mathcal K^l \rightarrow \infty$ as $l\rightarrow \infty$, thus updates of $a_b^{l+1}$ from the forward pass term vanishes.
    \item Bounded activation functions or unbounded activation functions with $\chi_{\mathcal K}^l < 1$: $\mathcal K^l \rightarrow \mathcal K^{\star}$, thus the contribution from the forward pass term is always $O(1)$.
\end{itemize}
Summarizing the discussion above, we do not have exploding and vanishing gradients problem originated from the term we introduced to tune the forward pass. The forward pass term simply speeds the update of $a_W^{l+1}$ and $a_b^{l+1}$ in most cases.

\paragraph{ReLU} Again we use a ReLU MLP network as an example. 

For $\sigma_b=0$, $\mathcal L_{\mathcal J \mathcal K \log}$ is equivalent to $(1+\lambda)\mathcal L_{\log}$ due to the scale invariance property of the ReLU activation function, which can be checked by using $\mathcal K^{l+1}(x,x) = \sigma_w^2 \mathcal K^l(x,x) / 2$.

For finite $\sigma_b$, we use $\mathcal K^{l,l+1}(x,x) = \mathcal J^{l,l+1} \mathcal K^l(x,x) + \sigma_b^2$:
\begin{align}\label{eqapp:relu_jkl}
    \mathcal L_{\mathcal J \mathcal K \log} = \frac{1}{2} \sum_{l=1}^{L} \left[\log(\mathcal J^{l, l+1})\right]^2 + \frac{\lambda}{2} \sum_{l=1}^{L} \left[\log\left (\mathcal J^{l,l+1} + \frac{\sigma_b^2}{\mathcal K^l(x,x)} \right)\right]^2 \,.
\end{align}

In ordered phase $\mathcal J^{l,l+1} = (a_W^{l+1}(t) \sigma_w)^2 / 2 < 1$, one can prove that $\mathcal K^l(x,x) \rightarrow \sigma_b^2 /(1 - (a_W^{l+1}(t)\sigma_w)^2/2)$ as $l \rightarrow \infty$, thus \eqref{eqapp:relu_jkl} is equivalent to JL at large depth.

In chaotic phase $\mathcal J^{l,l+1} = (a_W^{l+1}(t) \sigma_w)^2 / 2 > 1$ and $\mathcal K^l(x,x) \rightarrow \infty$. \eqref{eqapp:relu_jkl} is equivalent to JL with an extra overall factor $1+\lambda$ at large depth.
\section{ResMLP}

\subsection{Network Recursion Relation}

\paragraph{Input} The input image is chopped into an $N \times N$ grid of patches of size $P\times P$ pixels (often $16 \times 16$). The patches are fed into (the same) Linear layer to form a set of $N^2$ $d$-dimensional embeddings, referred to as channels. The resulting input to the ResMLP blocks : $h^0_{\mu i} \in \mathbb R^{N^2} \otimes \mathbb R^{d}$. Here and in what follows, Greek letters ($\mu,\nu$ etc.) index the patches, while Latin letters ($i,j$ etc.) index the channels. Note that in practice, the above two operations are combined into a Convolutional layer with the filer size coinciding with the patch-resolution ($P \times P \times C$); and the stride equal to $P$ so as to avoid overlap between patches. Here, $C$ is the number of channels in the original image.

\paragraph{ResMLP block} The input embedding $h^0_{\mu i}$ is passed through a series of ($L$) self-similar ResMLP blocks, which output $h^L_{\mu i} \in \mathbb R^{N^2} \otimes \mathbb R^{d}$. In the following, we use the notation $1^{2,3}_{4}$ for the parameters; where 1 denotes the parameter, 2 denotes the block-index, 3 denotes the specific action within the block, and 4 denotes the neural indices. A ResMLP block consists of the following operations.
\begin{align}\label{appeq:resmlp_rec_full}
    &\texttt{AffineNorm1:} & a^{l+1}_{\mu i} &= \left( \alpha^{{l+1},a}_i h^{l+1}_{\nu i} + \beta^{{l+1},a}_i \right) \,, & &\left( \mathbb R^{N^2} \otimes \mathbb R^d \right) \\
    &\texttt{linear1:} & b^{l+1}_{\mu i} &= \sum^{N^2}_{\nu=1} W^{{l+1},b}_{\mu\nu} a^{l+1}_{\nu i} + B^{{l+1},b}_{\mu} \,, & &\left( \mathbb R^{N^2} \otimes \mathbb R^d \right) \\
    &\texttt{residual1:} & c^{l+1}_{\mu i} &= \mathcal E^{{l+1},c}_i b^{l+1}_{\mu i} + \mu_1 a^{l+1}_{\mu i} \,, & &\left( \mathbb R^{N^2} \otimes \mathbb R^d \right) \\
    & \texttt{AffineNorm2:} & d^{l+1}_{\mu i} &= \left( \alpha^{{l+1},d}_i c^{l+1}_{\mu i} + \beta^{{l+1},d}_i \right) \,, &&\left( \mathbb R^{N^2} \otimes \mathbb R^d \right) \\
    &\texttt{linear2:} & e^{l+1}_{\mu i} &= \sum^d_{j=1} W^{{l+1},e}_{ij} d^{l+1}_{\mu j} + B^{{l+1},e}_{i} \,, & &\left( \mathbb R^{N^2} \otimes \mathbb R^{4d} \right) \\
    &\texttt{activation:} & f^{l+1}_{\mu i} &= \phi\left(e^{l+1}_{\mu j} \right) \,, & &\left( \mathbb R^{N^2} \otimes \mathbb R^{4d} \right) \\
    &\texttt{linear3:} & g^{l+1}_{\mu i} &= \sum^{4d}_{j=1} W^{{l+1},g}_{ij} f^{l+1}_{\mu j} + B^{{l+1},g}_{i}\,, & &\left( \mathbb R^{N^2} \otimes \mathbb R^d \right) \\
    &\texttt{residual2:} & h^{l+1}_{\mu i} &= \mathcal E^{{l+1},h}_i g^{l+1}_{\mu i} + \mu_2 c^{l+1}_{\mu i} \,, & &\left( \mathbb R^{N^2} \otimes \mathbb R^d \right)
\end{align}
where the brackets on the right contain the dimensions of the output of the layers.

We consider linear layers with weights and biases initialized with standard fan\_in. \texttt{linear1} acts on the patches, with parameters initialized as $W^{l+1,a}_{\mu\nu} \sim \mathcal N(0, \sigma_w^2/N)\,; B^{l+1,a}_{\mu} \sim \mathcal N(0, \sigma_b^2)$. \texttt{linear2} acts on the channels, with parameter initialized as $W^{l+1,e}_{ij} \sim \mathcal N(0, \sigma_w^2/\sqrt{d})\,; B^{l+1,e}_i \sim \mathcal N(0, \sigma_b^2)$. \texttt{linear3} also acts on the channels, with parameters initialized as $W^{l+1,g}_{ij} \sim \mathcal N(0, \sigma_w^2/\sqrt{4d})\,; B^{l+1,g}_i \sim \mathcal N(0, \sigma_b^2)$.
GELU is used as the activation function $\phi$.

\texttt{AffineNrom1} and \texttt{AffineNrom2} perform an element-wise multiplication with a trainable vector of weights $\alpha^{l+1,a}_i, \alpha^{l+1,d}_i  \in \mathbb R^d$ and an addition of a trainable bias vector $\beta^{l+1,a}_i, \beta^{l+1,d}_i \in \mathbb R^d$.
Residual branches are scaled by a trainable vector $\mathcal E^{l+1,c}_i, \mathcal E^{l+1,h}_i \in \mathbb R^{d}$ (\texttt{LayerScale}), whereas the skip connections are scaled by scalar strengths $\mu_1$ and $\mu_2$.

\paragraph{Output} The action of blocks is followed by an Average-Pooling layer, to to convert the output to a $d$-dimensional vector. This vector is fed into a linear classifier that gives the output of the network $h^{L+1}_i$.

\subsection{NNGP Kernel Recursion Relation}

At initialization, $\alpha^{l+1,a}_i = \alpha^{l+1,d}_i = \mathbf 1_d$ and $\beta^{l+1,a}_i = \beta^{l+1,d}_i = \mathbf 0_d$. Thus, AffineNorm layers perform identity operations at initialization.
\texttt{LayerScale} is initialized as $\mathcal E^{l+1,c}_i = \mathcal E^{l+1,h}_i = \mathcal E \, \mathbf 1_d$, where $\mathcal E$ is chosen to be a small scalar. (For examlpe, $\mathcal E$ is taken to be 0.1 for 12-block ResMLP and $10^{-5}$ for a 24-block ResMLP network.) Additionally, we also take $\mu_1 = \mu_2 = \mu$.

With these simplifications, we can obtain the recursion relation for the diagonal part of the Neural Network Gaussian Process (NNGP) kernel for the ResMLP block-outputs. We note that the the full NNGP kernel $\mathcal K^l_{\mu\nu;ij}$ is a tensor in $\mathbb R^{N^2} \otimes \mathbb R^{N^2} \otimes \mathbb R^{d} \times \mathbb R^{d}$. Here, we focus on its diagonal part $\mathcal K^l_{\mu\mu;ii}$. For clarity, we remove the subscripts ($\mu\mu;ii$). The diagonal part of the NNGP kernel for a block output $h^l_{\mu i}$ is defined as
\begin{align}\label{appeq:remlp_nngpk}
    \mathcal K^l \equiv \mathbb E_\theta \left[ h^l_{\mu i} h^l_{\mu i} \right] \,,
\end{align}
which is independent of its patch and channel indices, $\mu$ and $i$, in the infinite width limit. The recursion relation can be obtained by propagating the NNGP through a block. For clarity, we define NNGP kernel for the intermediate outputs within the blocks. For example, $\mathcal K^{l+1}_a \equiv \mathbb{E}_{\theta} \left[ a^{l+1}_{\mu i} a^{l+1}_{\mu i} \right]$, $\mathcal K^{l+1}_b \equiv \mathbb{E}_{\theta} \left[ b^{l+1}_{\mu i} b^{l+1}_{\mu i} \right]$, etc.
\begin{align}\label{appeq:resmlp_nngpk_rec}
    &\texttt{AffineNorm1:} & \mathcal K^{l+1}_a &= \mathcal K^l \,, \\
    \nonumber
    &\texttt{linear1:} & \mathcal K^{l+1}_b &= \sigma_w^2 \mathcal K^{l+1}_a + \sigma_b^2 \\
    &&&= \sigma_w^2 \mathcal K^l + \sigma_b^2 \,, \\
    \nonumber
    &\texttt{residual1:} & \mathcal K^{l+1}_c &= \mathcal E^2 \mathcal K^{l+1}_b + \mu^2 \mathcal K^l \\
    \nonumber
    &&&= \mu^2 \mathcal K^l + \mathcal E^2 \left( \sigma_w^2 \mathcal K^l + \sigma_b^2 \right) \\
    &&&= \left( \mu^2 + \mathcal E^2 \sigma_w^2 \right) \mathcal K^l + \mathcal E^2 \sigma_b^2 \,, \\
    \nonumber
    &\texttt{AffineNorm2:} & \mathcal K^{l+1}_d &= \mathcal K^{l+1}_c \\
    &&&= \mu^2 \mathcal K^l + \mathcal E^2 \left( \sigma_w^2 \mathcal K^l + \sigma_b^2 \right) \,, \\
    \nonumber
    &\texttt{linear2:} & \mathcal K^{l+1}_e &= \sigma_w^2 \mathcal K^{l+1}_d + \sigma_b^2 \\
    &&&= \sigma_w^2 \left( \mu^2 \mathcal K^l + \mathcal E^2 \left( \sigma_w^2 \mathcal K^l + \sigma_b^2 \right) \right) + \sigma_b^2 \,, \\
    \nonumber
    &\texttt{activation:} & \mathcal K^{l+1}_f &= \frac{\mathcal K^{l+1}_e}{4} + \frac{\mathcal K^{l+1}_e}{2\pi} \arcsin{\left( \frac{\mathcal K^{l+1}_e}{1 + \mathcal K^{l+1}_e} \right)} + \frac{\left( \mathcal K^{l+1}_e \right)^2}{\pi \left( 1 + \mathcal K^{l+1}_e \right) \sqrt{1 + 2\mathcal K^{l+1}_e}} \\
    \nonumber
    &&&\equiv \mathcal G \left[ \mathcal K^{l+1}_e \right] \\
    &&&= \mathcal G \left[ \sigma_w^2 \left( \mu^2 \mathcal K^l + \mathcal E^2 \left( \sigma_w^2 \mathcal K^l + \sigma_b^2 \right) \right) + \sigma_b^2 \right] \,, \\
    \nonumber
    &\texttt{linear3:} & \mathcal K^{l+1}_g &= \sigma_w^2 \mathcal K^{l+1}_f + \sigma_b^2 \\
    &&&= \sigma_w^2 \, \mathcal G \left[ \sigma_w^2 \left( \mu^2 \mathcal K^l + \mathcal E^2 \left( \sigma_w^2 \mathcal K^l + \sigma_b^2 \right) \right) + \sigma_b^2 \right] + \sigma_b^2 \,, \\
    \nonumber
    &\texttt{residual2:} & \mathcal K^{l+1} &= \mathcal E^2 \mathcal K^{l+1}_g + \mu^2 \mathcal K^{l+1}_c \\
    \nonumber
    &&&= \mathcal \mu^2 \left( \left( \mu^2 + \mathcal E^2 \sigma_w^2 \right) \mathcal K^l + \mathcal E^2 \sigma_b^2 \right) + \\
    \nonumber
    &&& \quad + \mathcal E^2 \, \left\{ \sigma_w^2 \, \mathcal G \left[ \sigma_w^2 \left( \mu^2 \mathcal K^l + \mathcal E^2 \left( \sigma_w^2 \mathcal K^l + \sigma_b^2 \right) \right) + \sigma_b^2 \right] + \sigma_b^2 \right\} \\
    \nonumber
    &&&= \left( \mu^4 + \mu^2 \mathcal E^2 \sigma_w^2 \right) \mathcal K^l + (1 + \mu^2) \mathcal E^2 \sigma_b^2 + \\
    &&&\quad + \mathcal E^2 \sigma_w^2 \, \mathcal  G \left[ \sigma_w^2 \left( \mu^2 \mathcal K^l + \mathcal E^2 \left( \sigma_w^2 \mathcal K^l + \sigma_b^2 \right) \right) + \sigma_b^2 \right] \,,
\end{align}
where we have defined 
\begin{equation}
    \mathcal G [z] =  \frac{z}{4} + \frac{z}{2\pi} \arcsin{\left( \frac{z}{1 + z} \right)} + \frac{\left( z \right)^2}{\pi \left( 1 + z \right) \sqrt{1 + 2z}} \,.
\end{equation}

Thus, we have a recursion relation, representing $\mathcal K^{l+1}$ in terms of $\mathcal K^l$.

As a side note, if we replace \texttt{GELU} activation function with \texttt{ReLU}, the relation simplifies greatly, offering us intuition. Specifically, $\mathcal G[z]$ gets replaced by $z/2$ in this case. This gives us the following recursion relation for ResMLP with \texttt{ReLU}.
\begin{align}
    \nonumber
    \mathcal K^{l+1} &= \left( \mu^4 + \mu^2 \mathcal E^2 \sigma_w^2 \right) \mathcal K^l + (1 + \mu^2) \mathcal E^2 \sigma_b^2 + \frac{1}{2} \mathcal E^2 \sigma_w^2 \left( \sigma_w^2 \left( \mu^2 \mathcal K^l + \mathcal E^2 \left( \sigma_w^2 \mathcal K^l + \sigma_b^2 \right) \right) + \sigma_b^2 \right) \\
    &= \left( \mu^4 + \mu^2 \mathcal E^2 \sigma_w^2 + \frac{1}{2} \mu^2 \mathcal E^2 \sigma_w^4 + \frac{1}{2} \mathcal E^4 \sigma_w^6 \right) \mathcal K^l + (1 + \mu^2 + \frac{1}{2}\sigma_w^2) \mathcal E^2 \sigma_b^2 + \frac{1}{2} \mathcal E^2 \sigma_w^2 \sigma_b^2
\end{align}

\subsection{Jacobian Recursion Relation}
Next, we calculate the APJN for ResMLP, between two consecutive blocks. For clarity, we first derive the expression for the partial derivative of ${l+1}^{th}$ block output $h^{l+1}_{\mu i}$ with respect to $l^{th}$ block output $h^l_{\nu j}$.
\begin{align}\label{appeq:resmlp_derivative}
    \nonumber
    \frac{\partial h^{l+1}_{\mu i}}{\partial h^l_{\nu j}} &= \mathcal E^{l+1,h}_i \frac{\partial g^{l+1}_{\mu i}}{\partial h^l_{\nu j}} + \mu_2 \frac{\partial c^{l+1}_{\mu i}}{\partial h^l_{\nu j}} \\
    \nonumber
    &= \mathcal E^{l+1,h}_i \sum_{k=1}^{4d} W^{l+1,g}_{ik} \frac{\partial f^{l+1}_{\mu k}}{\partial h^l_{\nu j}} + \mu_2 \frac{\partial c^{l+1}_{\mu i}}{\partial h^l_{\nu j}} \\
    \nonumber
    &= \mathcal E^{l+1,h}_i \sum_{k=1}^{4d} W^{l+1,g}_{ik} \phi'(e^{l+1}_{\mu k}) \frac{\partial e^{l+1}_{\mu k}}{\partial h^l_{\nu j}} + \mu_2 \frac{\partial c^{l+1}_{\mu i}}{\partial h^l_{\nu j}} \\
    \nonumber
    &= \mathcal E^{l+1,h}_i \sum_{k=1}^{4d} \sum_{m=1}^d W^{l+1,g}_{ik} \phi'(e^{l+1}_{\mu k}) W^{l+1,e}_{km} \frac{\partial d^{l+1}_{\mu m}}{\partial h^l_{\nu j}} + \mu_2 \frac{\partial c^{l+1}_{\mu i}}{\partial h^l_{\nu j}} \\
    \nonumber
    &= \mathcal E^{l+1,h}_i \sum_{k=1}^{4d} \sum_{m=1}^d W^{l+1,g}_{ik} \phi'(e^{l+1}_{\mu k}) W^{l+1,e}_{km} \frac{\partial c^{l+1}_{\mu m}}{\partial h^l_{\nu j}} + \mu_2 \frac{\partial c^{l+1}_{\mu i}}{\partial h^l_{\nu j}} \\
    \nonumber
    &= \mathcal E^{l+1,h}_i \sum_{k=1}^{4d} \sum_{m=1}^d W^{l+1,g}_{ik} \phi'(e^{l+1}_{\mu k}) W^{l+1,e}_{km} \mathcal E^{l+1,c}_m \frac{\partial b^{l+1}_{\mu m}}{\partial h^l_{\nu j}} + \\
    \nonumber 
    &\quad + \mathcal E^{l+1,h}_i \sum_{k=1}^{4d} \sum_{m=1}^d W^{l+1,g}_{ik} \phi'(e^{l+1}_{\mu k}) W^{l+1,e}_{km} \mu_1 \frac{\partial a^{l+1}_{\mu m}}{\partial h^l_{\nu j}} + \mu_2 \mathcal E^{l+1,c}_i \frac{\partial b^{l+1}_{\mu i}}{\partial h^l_{\nu j}} + \mu_2\mu_1 \frac{\partial a^{l+1}_{\mu i}}{\partial h^l_{\nu j}} \\
    \nonumber
    &= \mathcal E^{l+1,h}_i \sum_{k=1}^{4d} \sum_{m=1}^d \sum_{\lambda=1}^{N^2} W^{l+1,g}_{ik} \phi'(e^{l+1}_{\mu k}) W^{l+1,e}_{km} \mathcal E^{l+1,c}_m W^{l+1,b}_{\mu\lambda} \frac{\partial a^{l+1}_{\lambda m}}{\partial h^l_{\nu j}} + \\
    \nonumber
    &\quad + \mathcal E^{l+1,h}_i \sum_{k=1}^{4d} \sum_{m=1}^d W^{l+1,g}_{ik} \phi'(e^{l+1}_{\mu k}) W^{l+1,e}_{km} \mu_1 \delta_{\mu\nu} \delta_{mj} + \\
    \nonumber
    &\quad + \mu_2 \mathcal E^{l+1,c}_i \sum_{\lambda=1}^{N^2} W^{l+1,b}_{\mu\lambda} \frac{\partial a^{l+1}_{\lambda i}}{\partial h^l_{\nu j}} + \mu_2\mu_1 \delta_{\mu\nu} \delta_{ij} \\
    \nonumber
    &= \mathcal E^{l+1,h}_i \mathcal E^{l+1,c}_j \sum_{k=1}^{4d} W^{l+1,g}_{ik} \phi'(e^{l+1}_{\mu k}) W^{l+1,e}_{kj} W^{l+1,b}_{\mu\nu} + \\
    \nonumber
    &\quad + \mu_1 \mathcal E^{l+1,h}_i \delta_{\mu\nu} \sum_{k=1}^{4d} W^{l+1,g}_{ik} \phi'(e^{l+1}_{\mu k}) W^{l+1,e}_{kj} + \mu_2 \mathcal E^{l+1,c}_i \delta_{ij} W^{l+1,b}_{\mu\nu} + \mu_2\mu_1 \delta_{\mu\nu} \delta_{ij} \\
    \nonumber
    &= \mathcal E^2 \sum_{k=1}^{4d} W^{l+1,g}_{ik} \phi'(e^{l+1}_{\mu k}) W^{l+1,e}_{kj} W^{l+1,b}_{\mu\nu} + \mu \mathcal E \delta_{\mu\nu} \sum_{k=1}^{4d} W^{l+1,g}_{ik} \phi'(e^{l+1}_{\mu k}) W^{l+1,e}_{kj} + \\
    &\quad + \mu\mathcal E \delta_{ij} W^{l+1,b}_{\mu\nu} + \mu^2 \delta_{\mu\nu} \delta_{ij} \,,
\end{align}
where in the last step, we have used the initial values of the parameters : $\mathcal E^{l+1,c}_i = \mathcal E^{l+1,h}_i = \mathcal E \, \mathbf 1_d$ and $\mu_1 = \mu_2 = \mu$.

Next, we calculate the APJN using \eqref{appeq:resmlp_derivative}. We will perform the calculation in the limit of large $N^2$ and $d$; dropping all the corrections of order $\frac{1}{N^2}$ and $\frac{1}{d}$.

\begin{align}\label{appeq:resmlp_apjn}
    \nonumber
    \mathcal J^{l,l+1} &= \frac{1}{N^2 d} \mathbb E_\theta \left[ \sum_{\mu,\nu}^{N^2} \sum_{i,j}^d \frac{\partial h^{l+1}_{\mu i}}{\partial h^l_{\nu j}} \frac{\partial h^{l+1}_{\mu i}}{\partial h^l_{\nu j}} \right] \\
    \nonumber
    &= \frac{1}{N^2 d} \mathbb E_\theta \left[ \mathcal E^4 \sum_{\mu,\nu}^{N^2} \sum_{i,j}^d \sum_{k,m}^{4d} W^{l+1,g}_{ik} W^{l+1,g}_{im} \phi'(e^{l+1}_{\mu k}) \phi'(e^{l+1}_{\mu m}) W^{l+1,e}_{kj} W^{l+1,e}_{mj} W^{l+1,b}_{\mu\nu} W^{l+1,b}_{\mu\nu} + \right. \\
    \nonumber
    &\qquad\qquad\quad \left. + \mu^2\mathcal E^2 \sum_{\mu,\nu}^{N^2} \sum_{i,j}^d \sum_{k,m}^{4d} \delta_{\mu\nu} W^{l+1,g}_{ik} W^{l+1,g}_{im} \phi'(e^{l+1}_{\mu k}) \phi'(e^{l+1}_{\mu m}) W^{l+1,e}_{kj} W^{l+1,e}_{mj} + \right. \\
    \nonumber
    &\qquad\qquad\quad \left. + \mu^2\mathcal E^2 \sum_{\mu,\nu}^{N^2} \sum_{i,j}^d \delta_{ij} W^{l+1,b}_{\mu\nu} W^{l+1,b}_{\mu\nu} + \mu^4 \sum_{\mu,\nu}^{N^2} \sum_{i,j}^d \delta_{\mu\nu} \delta_{ij} \right] \\
    \nonumber
    &= \frac{1}{N^2 d} \mathbb E_\theta \left[ \mathcal E^4 \sigma_w^2 \sum_{\mu,\nu}^{N^2} \sum_j^d \sum_k^{4d} \phi'(e^{l+1}_{\mu k}) \phi'(e^{l+1}_{\mu k}) W^{l+1,e}_{kj} W^{l+1,e}_{kj} W^{l+1,b}_{\mu\nu} W^{l+1,b}_{\mu\nu} +  \right. \\
    \nonumber
    &\qquad\qquad\quad \left. + \mu^2\mathcal E^2 \sigma_w^2 \sum_{\mu,\nu}^{N^2} \sum_j^d \sum_k^{4d} \delta_{\mu\nu} \phi'(e^{l+1}_{\mu k}) \phi'(e^{l+1}_{\mu k}) W^{l+1,e}_{kj} W^{l+1,e}_{kj} + \right. \\
    \nonumber
    &\qquad\qquad\quad + \left. \mu^2\mathcal E^2 \sigma_w^2 N^2 d + \mu^4 N^2 d \right] \\
    \nonumber
    &= \frac{1}{N^2 d} \mathbb E_\theta \left[ \frac{1}{4} \mathcal E^4 \sigma_w^6 \sum_\mu^{N^2} \sum_k^{4d} \phi'(e^{l+1}_{\mu k}) \phi'(e^{l+1}_{\mu k}) + \mu^2\mathcal E^2 \sigma_w^4 \sum_\mu^{N^2} \sum_k^{4d} \phi'(e^{l+1}_{\mu k}) \phi'(e^{l+1}_{\mu k}) + \right. \\
    \nonumber
    &\qquad\qquad\quad \left. + (\mathcal E^2 \sigma_w^2 + \mu^2) \mu^2 N^2 d \right] \\
    \nonumber
    &= (\mu^2 + \mathcal E^2 \sigma_w^2) \left( \mu^2 + \mathcal E^2 \sigma_w^4 \mathbb E_\theta \left[ \phi'(e^{l+1}_{\mu k}) \phi'(e^{l+1}_{\mu k}) \right] \right) \\
    \nonumber
    &= (\mu^2 + \mathcal E^2 \sigma_w^2) \left( \mu^2 + \mathcal E^2 \sigma_w^4 \, \mathcal H_e [\mathcal K^{l+1}_e] \right) \\
    &= (\mu^2 + \mathcal E^2 \sigma_w^2) \left( \mu^2 + \mathcal E^2 \sigma_w^4 \, \mathcal H [\mathcal K^l] \right) \,,
\end{align}
where we have defined 
\begin{align}
    \nonumber
    \mathcal H_e [\mathcal K^{l+1}_e] &\equiv \mathbb{E}_\theta \left[ \phi'(e^{l+1}_{\mu k}) \phi'(e^{l+1}_{\mu k}) \right] \\
    &= \frac{1}{4} + \frac{1}{2\pi} \left( \arcsin{\left( \frac{\mathcal K^{l+1}_e}{1 + \mathcal K^{l+1}_e} \right)} + \frac{\mathcal K^{l+1}_e (3  + 5 \mathcal K^{l+1}_e)}{(1 + \mathcal K^{l+1}_e) (1 + 2\mathcal K^{l+1}_e)^{3/2}} \right) \,.
\end{align}
We also write $\mathcal K^{l+1}_e$ in terms of $\mathcal K^l$ and define
\begin{align}
    \nonumber
    \mathcal H [\mathcal K^l] &= \mathcal H_e [\mathcal K^{l+1}_e] \\
    &= \mathcal H_e \left[ \sigma_w^2 \left( \mu^2 \mathcal K^l + \mathcal E^2 (\sigma_w^2 \mathcal K^l + \sigma_b^2) + \sigma_b^2 \right) \right] \,.
\end{align}

It is clear from \eqref{appeq:resmlp_apjn} that for $\mu=1$, $\mathcal J^{l,l+1} > 1$, rendering the network off criticality. However, $\mathcal J^{l,l+1}$ can be tuned arbitrarily close to criticality by taking $\mathcal E$ to be small at $t=0$. This explains the necessity for \texttt{LayerScale} with small initial value in the ResMLP architecture.

We note that the results in \eqref{appeq:resmlp_apjn} greatly simplify on using \texttt{ReLU} instead of \texttt{GELU} as $\phi$. We mention them here to provide intuition. $\mathcal H_e [\mathcal K^{l+1}_e] = \mathbb{E}_\theta \left[ \phi'(e^{l+1}_{\mu k}) \phi'(e^{l+1}_{\mu k}) \right] = \frac{1}{2}$ in this case. This gives us the simple result
\begin{equation}
    \mathcal J^{l,l+1} = (\mu^2 + \mathcal E^2 \sigma_w^2) \left( \mu^2 + \frac{1}{2}\mathcal E^2 \sigma_w^4 \right) 
\end{equation}
for \texttt{ReLU} activation function.